\documentclass[Afour,sageh,times]{sagej}
\usepackage{moreverb,url}
\usepackage{natbib}
\usepackage[colorlinks,bookmarksopen,bookmarksnumbered,citecolor=red,urlcolor=red]{hyperref}
\usepackage{amsmath}
\usepackage{amssymb}
\usepackage{mathtools}
\usepackage{balance}
\usepackage{graphics}
\usepackage{psfrag}

\newcommand\BibTeX{{\rmfamily B\kern-.05em \textsc{i\kern-.025em b}\kern-.08em
T\kern-.1667em\lower.7ex\hbox{E}\kern-.125emX}}

\usepackage{latexsym}
\newtheorem{proposition}{Proposition}
\newtheorem{lemma}{Lemma}
\newtheorem{result}{Preliminary result}
\usepackage{xspace}
\newcommand{\dd}[0]{\mbox{2{\sc d}}\xspace}
\newcommand{\ddd}[0]{\mbox{3{\sc d}}\xspace}

\newcommand{\matx}[1]{\mathbf{\uppercase{#1}}}
\newcommand{\quat}[1]{\mathsf{\lowercase{#1}}}
\newcommand{\vect}[1]{\mathbf{\lowercase{#1}}}

\newcommand{\id}[1]{\matx{I}_{#1}}
\newcommand{\zero}[2]{\matx{0}_{#1 \times #2}}

\newcommand{\rx}[0]{{\matx{R}_x}}
\newcommand{\trx}[0]{{\tilde{\matx{R}}_x}}
\newcommand{\tx}[0]{{\vect{t}_x}}
\newcommand{\ra}[0]{{\matx{R}_a}}
\newcommand{\ta}[0]{{\vect{t}_a}}
\newcommand{\ua}[0]{{\vect{u}_a}}
\newcommand{\rb}[0]{{\matx{R}_b}}
\newcommand{\tb}[0]{{\vect{t}_b}}

\begin{document}
\runninghead{Andreff, Horaud and Espiau}
\title{Robot Hand-Eye Calibration using Structure-from-Motion}
%\thanks{This work was supported by the European 
%  Community through the Esprit-IV reactive LTR project number 26247
%(VIGOR).
%During this work, Nicolas Andreff was staying at INRIA
%Rh\^one-Alpes as a Ph.D. student.}
\author{Nicolas Andreff\affilnum{1}, Radu Horaud\affilnum{2} and Bernard Espiau\affilnum{2}}
\affiliation{\affilnum{1}Institut Fran\c{c}ais de M\'ecanique Avanc\'ee,
BP 265, 63175 Aubi\`ere Cedex, France.
\affilnum{2}INRIA Rh\^one-Alpes and GRAVIR-IMAG,
655, av. de l'Europe, 38330 Montbonnot Saint Martin, France.}

\begin{abstract}
In this paper we propose a new flexible method for hand-eye calibration.
The vast majority of existing hand-eye calibration techniques requires a
calibration rig which is used in conjunction with camera pose estimation
methods. Instead, we combine structure-from-motion with known robot
motions and we show that the solution can be obtained in linear form.
The latter solves for both the hand-eye parameters and for the unknown
scale factor inherent with structure-from-motion methods. The algebraic
analysis that is made possible with such a linear formulation allows to
investigate not only the well known case of general screw motions but
also such singular motions as pure translations, pure rotations, and
planar motions. In essence, the robot-mounted camera looks to an unknown
rigid layout, tracks points over an image sequence and estimates the
camera-to-robot relationship. Such a self calibration process is
relevant for unmanned vehicles, robots working in remote places, and so
forth. We conduct a large number of experiments which validate the
quality of the method by comparing it with existing ones.
\end{abstract}
%\begin{keyword}
%Self-calibration, structure-from-motion, Sylvester equation
%\end{keyword}

\maketitle

\section{Introduction}
The background of this work is the guidance of a robot by visual
servoing~\citep{espiau92,mathese}. In this framework, a basic issue is
to determine the spatial relationship
between a camera mounted onto a robot end-effector (Fig.~\ref{tete})
and the end-effector itself. This spatial 
relationship is a rigid transformation, a rotation and a
translation, known as the hand-eye transformation.
The determination of this transformation is called hand-eye
calibration.

The goal of this paper is to describe a technique allowing the hand-eye
calibration to be performed in the working site. In practice, this
requires that: 
\begin{itemize}
\item No calibration rig will be allowed.

A calibration rig is a very accurately manufactured \ddd object
holding targets as visual features.
Mobile robots and space applications of robotics are typical examples
where a calibration rig cannot be used. During their mission, such
robots may nevertheless need to be calibrated again. However, as
affordable on-board weight is limited, they 
can not carry a calibration object and should use their surrounding
environment instead. Thus, the availability of a hand-eye
self-calibration method is mandatory.

\item Special and/or large motions are difficult to achieve and hence
should be avoided. 

Indeed, since the hand-eye system must be calibrated on-site, the
amount of free robot workspace is limited and the motions
have therefore to be of small amplitude. Therefore, the
self-calibration method must be able to handle  
a large variety of motions, including small ones.

\end{itemize}

\begin{figure}[t]
\centerline{\rotatebox{-90}{\resizebox{!}{7cm}{\includegraphics{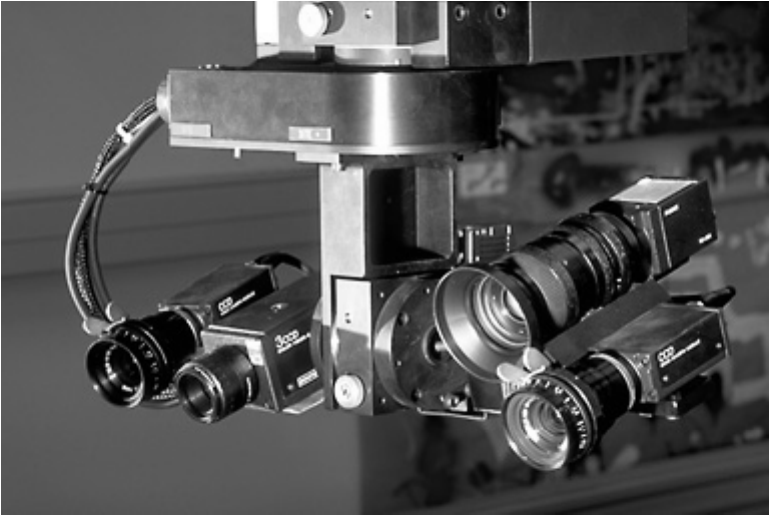}}}}
\caption{Some cameras mounted on our 5 DOF robot.}
\label{tete} 
\end{figure}

Hand-eye calibration was first studied a 
decade ago~\citep{tsai89,shiu89a}.
It was shown that any solution to the problem requires to consider
both Euclidean end-effector motions and camera motions\footnote{Notice
that this requirement may be implicit as in~\citep{remy97}.}.
While the end-effector motions can be obtained from the 
encoders, the camera motions are to be computed from the images.  
It was also shown, both
algebraically~\citep{tsai89} and geometrically~\citep{chen91c}, that
a sufficient condition to the uniqueness of the solution is the
existence of two calibration motions with non-parallel rotation axes. 

Several methods were
proposed~\citep{tsai89,daniilidis96b,horaud95e,shiu89a,chou91a,wang92b}
to solve for hand-eye calibration under 
the assumption that both end-effector and camera motions were known.
All these methods represent rotations by a minimal parameterization
and use pose algorithms to estimate the camera position relatively to
the fixed calibration rig.
%They differ by the way they represent Euclidean motions, but
%all have two points in common: 
%(i) rotation is represented by a minimal parameterization and
%(ii) all proposed methods use pose algorithms to estimate the camera
%motion relatively to the fixed calibration rig. 
Recall that pose algorithms
require the \ddd Euclidean coordinates of the rig targets to be known
together with their associated \dd projections onto each image.

%Moreover, as the proposed methods
%use reduced representations of the rotation and since these are
%ill-defined when rotation angle is small, the calibration motions
%must be as large as possible: a rule for such a choice of large
%calibration motions is even given in~\citep{tsai89}. 

Another approach is proposed by Wei {\em et al.}~\citep{wei98a}, who
perform simultaneously 
hand-eye calibration and camera calibration without any calibration
rig. However, this requires a complex non-linear
minimization and the use of a restrictive class of calibration
motions. Moreover, no algebraic analysis of the problem is given.

With regard to the existing approaches, we propose a new
hand-eye self-calibration method which 
exploits two main ideas. The first idea is that a
specific algebraic treatment is necessary to handle small
rotations. Indeed, minimal representations of rotation may be either
singular (Euler angles, axis/angle) or non unique (quaternions) for
specific configurations. Moreover, when the rotation angle is small,
it is hard to accurately estimate the rotation axis. Yet, existing
approaches need to estimate the latter. Hence, one should use, as
far as hand-eye calibration is concerned, as large calibration motions 
as possible: a rule for such a choice of large
calibration motions is even given in~\citep{tsai89}. Instead,
we prefer to use orthogonal matrices, that are always defined, even
though we need either to estimate more parameters (9 instead of~3
to~4) or to take into account the orthogonality constraints.

%, since minimal parameterizations of rotations are not defined
%for small angles and are therefore ill-conditioned.
The second idea is that camera motion can be 
computed from structure-from-motion algorithms rather than from pose
algorithms. Structure-from-motion estimation methods give very accurate results in
terms of camera motion parameters and a variety of techniques exist from
factorization methods to non-linear bundle-adjustment. From a practical
point of view an iterative factorization method such as the one
described in \citep{christy96a} achieves a good compromise between
purely linear methods and highly non-linear ones. 

Our contributions can be summarized in the following. Firstly,
hand-eye calibration 
is reformulated in order to take into account the estimation of
camera motions from structure-from-motion algorithms. Indeed, camera
motions are thus obtained up to an unknown scale factor, which is
introduced in the formulation.
Secondly, a linear
formulation, based on the representation of rotations by orthogonal
matrices, is proposed which enables small  
calibration motions. It allows us to give a common framework to
the (already solved) general motion case as well as to singular motion cases
(where the general solution fails) such as: pure translations, pure
rotations, and planar motions. 
Thirdly, an algebraic study of this linear formulation
is performed which shows that partial calibration can nevertheless be
performed when the sufficient condition for the uniqueness of the
solution is not fulfilled. Fourthly, in-depth experiments are conducted
with comparison to other methods. They show that the use of
structure-from-motion does not affect the numerical
accuracy of our method with respect to existing ones.

The remainder of this paper is organized as
follows. Section~\ref{classic} recalls the classical formulation of
hand-eye calibration and the structure-from-motion paradigm. 
Section~\ref{new} gives contains the formulation of the linear
hand-eye self-calibration method.
Section~\ref{analysis} contains its algebraic analysis.
Finally, Section~\ref{expe} 
gives some experimental results and Section~\ref{conclusion} concludes
this work. 

\section{Background}
\label{classic}

In this section, after defining the notation used in this article,
we briefly present the classical formulation of hand-eye
calibration with a short description of three methods that will be
used as references in the experimental section (Section~\ref{expe}).
We then describe the estimation of camera motions, concluding
in favor of Euclidean reconstruction rather than pose computation.

\subsection{Notation}

Matrices are represented by upper-case bold-face
letters (e.g. $\matx{R}$) and vectors by lower-case bold-face letters
(e.g. $\vect{t}$). 

Rigid transformations (or, equivalently, Euclidean
motions) are represented with homogeneous matrices of the form:
\[
\begin{pmatrix}
\matx{R}&\vect{t}\\
0\ 0\ 0& 1
\end{pmatrix}
\]
where $\matx{R}$ is a $3\times 3$ rotation matrix and $\vect{t}$ is a
$3\times 1$ translation vector. This 
rigid transformation will be often referred to as the couple
$(\matx{R},\vect{t})$. 

In the linear formulation of the problem, we will use the linear
operator $vec$
and the tensor product, also known as Kronecker product.
The $vec$ operator was introduced in \citep{Neudecker69} and
reorders (one line after the other) the coefficients of a $(m \times
n)$ matrix $\matx{M}$ into the $mn$ vector
\[
vec(\matx{M}) = (M_{11}, \ldots , M_{1n}, M_{21}, \ldots, M_{mn})^T\]
The Kronecker product~\citep{Bellman60,Brewer78} is noted
$\otimes$. From two matrices $\matx{M}$ 
and $\matx{N}$ with respective dimensions $(m \times n)$ and $(o
\times p)$, it defines the resulting $(mo \times np)$ matrix:
\begin{equation}
\matx{M} \otimes \matx{N} = \left(
\begin{matrix}
M_{11} \matx{N} & \ldots & M_{1n} \matx{N}\\
\vdots & \ddots & \vdots \\
M_{m1} \matx{N} & \ldots & M_{mn} \matx{N}
\end{matrix}
\right)
\end{equation}

\subsection{Hand-eye problem formulation}
We present here the classical
approach~\citep{tsai89,chen91c,daniilidis96b,horaud95e,shiu89a,chou91a,wang92b}
which states that, when the camera undergoes a motion
$\matx{A}=(\ra,\ta)$ and that the corresponding end-effector motion
is $\matx{B}=(\rb,\tb)$, then they are conjugated by the
hand-eye transformation $\matx{X}=(\rx,\tx)$ (Fig.~\ref{ax=xb.fig}).
This yields the following homogeneous matrix equation:
\begin{equation}
\matx{A} \matx{X} = \matx{X} \matx{B}
\label{ax=xb.eq}
\end{equation}
where $\matx{A}$ is estimated, $\matx{B}$ is assumed to be known and
$\matx{X}$ is the unknown.

\begin{figure}[t!]
\centerline{\resizebox{!}{4cm}{\includegraphics{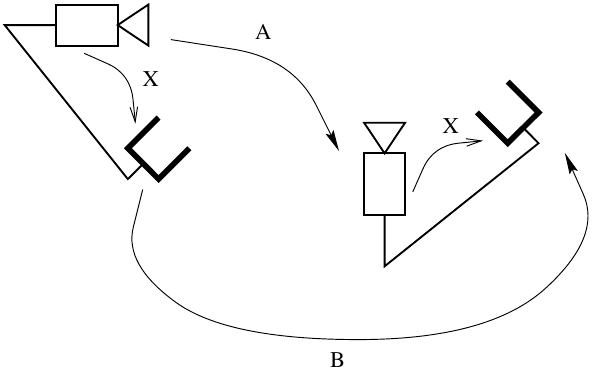}}}
\caption{End-effector (represented here by a
gripper) and camera motions are conjugated by the hand-eye
transformation $\matx{X}$.}
\label{ax=xb.fig} 
\end{figure}

Equation~(\ref{ax=xb.eq}),
applied to each motion~$i$, splits into:
\begin{eqnarray}
\ra_i \rx &=& \rx \rb_i \label{tsai1}\\
\ra_i \tx + \ta_i  &=&  \rx \tb_i + \tx \label{tsai2}
\end{eqnarray}

In the method proposed in~\citep{tsai89}, the first equation is solved
by least-square minimization of a linear system obtained by using the
axis/angle representation of the 
rotations. Once $\rx$ is known, the second equation is also solved
with linear least squares techniques. 

To avoid this two-stage solution which propagates the error on the
rotation estimation onto the translation, a
non-linear minimization method based on the representation
of the rotations with unit quaternions was proposed in~\citep{horaud95e}. 
Similarly, a method based on the unit dual quaternion
representation of Euclidean motions was developed in~\citep{daniilidis96b} to
solve simultaneously  
for hand-eye rotation and hand-eye translation.

\subsection{Computing the camera motions}

In the prior
work~\citep{tsai89,chen91c,daniilidis96b,horaud95e,shiu89a,chou91a,wang92b},
camera motions were computed considering images one  
at a time, as follows.
First, \dd-to-\ddd correspondences were established between the \ddd
targets on the calibration rig and their \dd projections onto 
each image~$i$. Then, from the \ddd coordinates of the targets, their
\dd projections and the intrinsic camera
parameters, the pose (i.e. position and orientation) of the
camera with respect to the
calibration rig is estimated $\matx{P}_i=(\matx{R}_i, \vect{t}_i)$.
Finally, the camera motion between image
$i-1$ and image $i$ $\matx{A}_i=(\ra_i, \ta_i)$
is hence obtained by simple composition (Fig.~\ref{pose}): 
\[
\matx{A}_i=(\matx{R}_{i}\matx{R}_{i-1}^T,
\vect{t}_i-\matx{R}_{i}\matx{R}_{i-1}^T\vect{t}_{i-1})
\]

Alternatively, one may simultaneously consider all the images that were
collected during camera motion. Thus, one may use 
the multi-frame structure-from-motion paradigm (see~\citep{oliensis97a} for 
a review). The advantage of structure-from-motion over pose algorithms 
is that the former does not require any knowledge about the observed
\ddd object. Indeed, structure-from-motion only
relies on \dd-to-\dd correspondences. 
The latter are easier to obtain
since they depend on image information only.
There are two classes of \mbox{(semi-)automatic} methods to find them: a
discrete approach, known as {\em matching}~\citep{gruen85a}, and a
continuous approach, known as {\em tracking}~\citep{hager99a}. 

\begin{figure}[t!]
\psfrag{calibration block}{calibration rig}
\psfrag{gripper}{gripper}
\psfrag{camera}{camera}
\psfrag{X}{$\matx{X}$}
\psfrag{P1}{$\matx{P}_{i-1}$}
\psfrag{P2}{$\matx{P}_i$}
\psfrag{Pn}{$\matx{P}_n$}
\psfrag{A1}{$\matx{A}_i$}
\psfrag{B1}{$\matx{B}_i$}
\centerline{\resizebox{!}{6.5cm}{\includegraphics{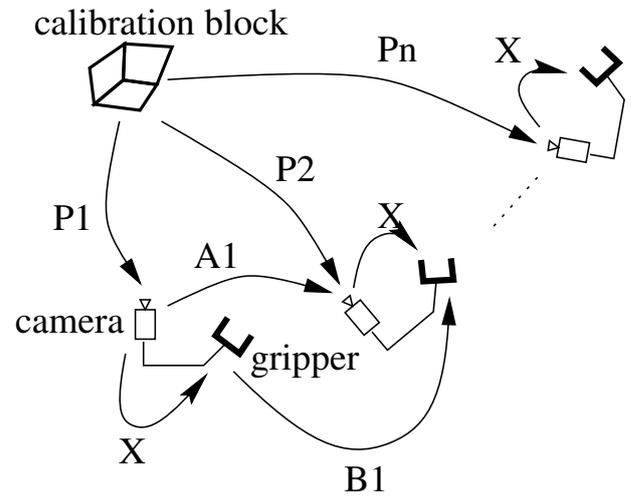}}}
\caption{Hand-eye calibration from pose estimation}
\label{pose} 
\end{figure}

%In general, structure-from-motion algorithms are used to reconstruct
%the \ddd structure of an unknown scene viewed from several
%images. However, a by-product of such algorithms is some knowledge of 
%the relative motions of the camera between the positions where the
%images were taken. For instance, with an uncalibrated camera (i.e. one 
%does not its intrinsic parameters), one can reconstruct projective
%camera motions~\citep{faugeras93a,shashua94d}.

A relevant class of structure-from-motion methods is known as Euclidean
reconstruction~\citep{cui94a,taylor91a,polman94a,tomasi91a,konderink91a,christy96a}.
It assumes that the camera is calibrated (i.e. the camera intrinsic
parameters are known). From this knowledge, one can reconstruct the 
structure of the scene and the motion of the camera up to an unknown
scale factor (Fig.~\ref{ambiguite}) using various methods (see below). 
This unknown scale factor is a global scale factor associated with the fact that
rigidity is defined up to a similitude and is the same for all the camera
motions in the sequence (Fig.~\ref{ambiguite2}).
Therefore, the estimated camera motions are of the form:
\begin{equation}
\matx{a}_i(\lambda) = 
\begin{pmatrix}
\ra_i&\lambda \ua{}_i\\
0\ 0\ 0&1
\end{pmatrix}
\label{a_lambda}
\end{equation}
where $\ra_i$ is the
rotation of the camera between image $i-1$ and image $i$, $\lambda$
is the unknown scale factor and $\ua{}_i$ is a vector, parallel to the 
camera translation $\ta{}_i$ and such that 
\begin{equation}
\ta_i=\lambda \ua_i
\label{ua}
\end{equation}
Taking, without loss of generality, the first motion as a motion
with non zero translation allows to arbitrarily choose $\ua_1$ as a unit 
vector. Hence, $\lambda = \| \ta_1 \|$.
Consequently, the $\ua_i$'s are related by: 
$ \ua_i = \ta_i/\|\ua_1\|$ and $\lambda$ can be
interpreted as the unknown norm of the first translation.

In summary, camera rotations are completely recovered while camera
translations are recovered up to a single unknown scale factor. 

\begin{figure}[t!]
\centerline{%
   \resizebox{!}{5cm}{%
        \includegraphics{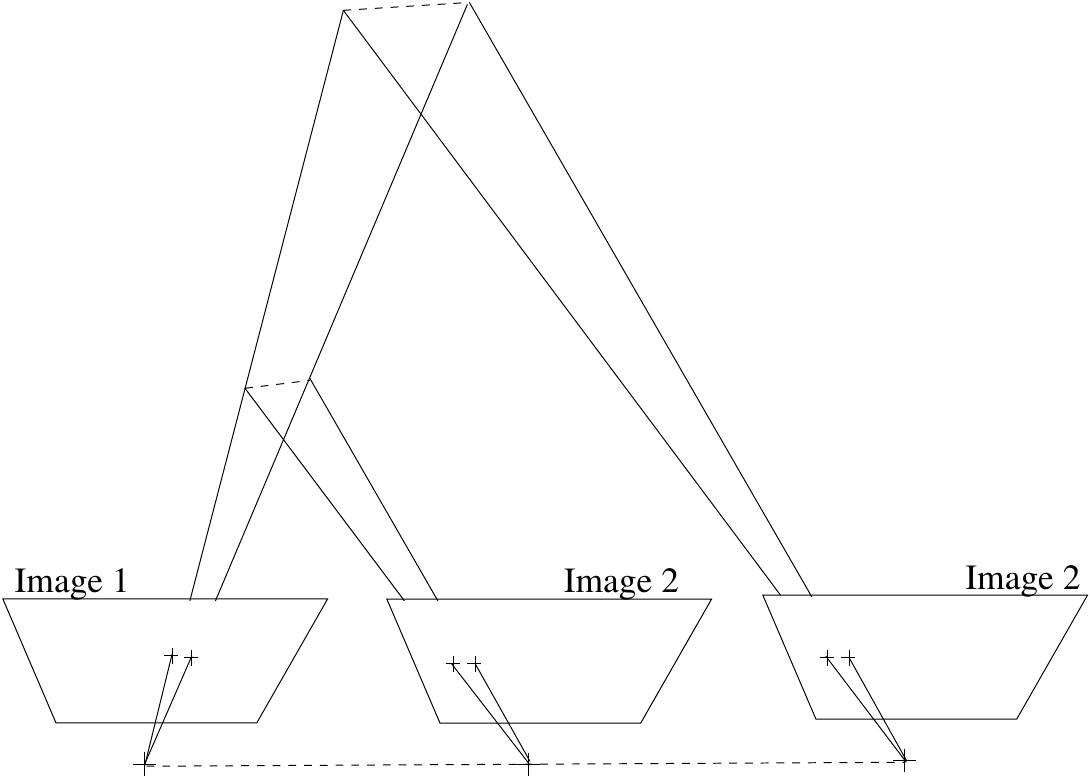}
   }
}
\caption{\label{ambiguite} Given two images, structure-from-motion methods determine
the direction of the baseline between the two camera positions, but
not the baseline length. Thus, the size of the reconstructed
object is not determined since it is proportional to the unknown baseline
length.}
\end{figure}

\begin{figure}[h!]
\centerline{%
   \resizebox{!}{5cm}{%
        \includegraphics{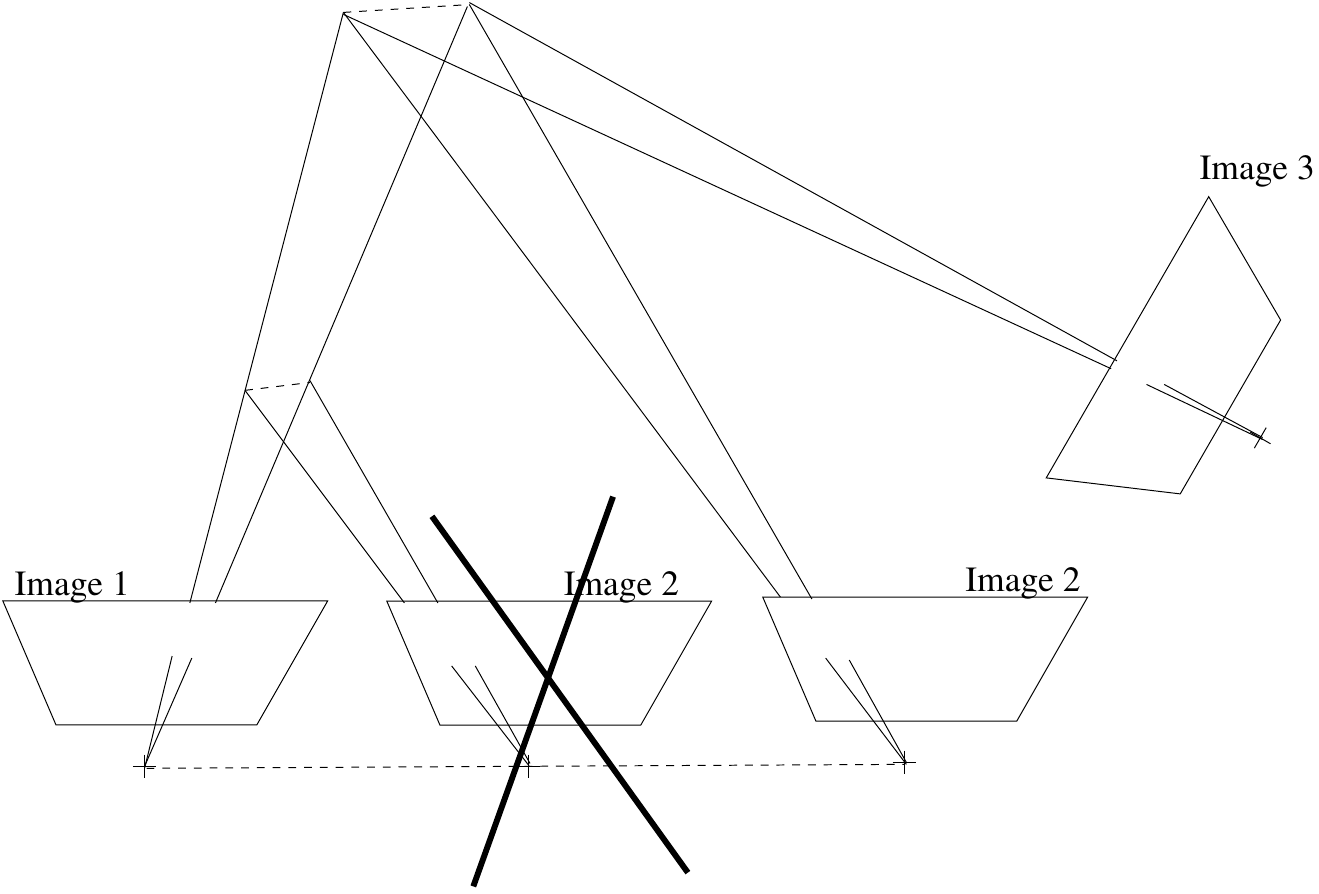}
   }
}
\caption{\label{ambiguite2}  Once the scale factor is resolved between
the first two images, the second camera position is uniquely defined
with respect to the first one and consequently, the following camera
positions are also uniquely defined.}
\end{figure}

In practice, which structure-from-motion algorithm should we chosen?
Affine camera models~\citep{polman94a,tomasi91a,konderink91a} yield
simple linear solutions to the
Euclidean reconstruction problem, based on matrix factorization.
However, affine models are first-order approximations of the
perspective model. Hence, only approximations of the 
Euclidean camera motions can be obtained.
Besides, solutions exist~\citep{cui94a,taylor91a}, based on
the perspective 
model, that offer some Euclidean information on the camera motions,
but are non linear. 
Our choice lies in fact between these two classes of
methods: we propose a method for Euclidean
reconstruction by successive affine approximations of the perspective
model~\citep{christy96a}, which combines the simplicity of affine methods and
the accuracy of non linear methods.

In summary, in order to estimate camera motions, structure-from-motion
methods are more flexible than
pose computation methods, since no \ddd model is needed. The drawback of
lowering this constraint is 
that camera motions are estimated up to an unknown scale factor which
we must take into account in the hand-eye self-calibration method.

\section{A new linear formulation}
\label{new}

In this section, we first modify the formulation of hand-eye
calibration in order to take into account the use of Euclidean reconstruction
to compute camera motions. Then, we give a solution to this problem
which handles small rotations.

\subsection{Using structure-from-motion}

For using structure-from-motion to estimate camera motions, we have to
take into account the unknown scale factor $\lambda$.
Indeed, the homogeneous equation (\ref{ax=xb.eq})  becomes (compare
Fig.~\ref{pose} and Fig.~\ref{rec3d}):
\begin{equation}
\matx{A}_i(\lambda) \matx{X}=\matx{X}\matx{B}_i
\end{equation}
where $\matx{A}_i(\lambda)$ is the $i$th estimated camera
motion.
From~(\ref{a_lambda}) and~(\ref{ua}), we thus obtain a set of two
equations, similar to~(\ref{tsai1})--(\ref{tsai2}):
\begin{eqnarray}
\ra_i \rx &=& \rx \rb_i \label{ax=xb.rot}\\
\ra_i \tx + \lambda \ua{}_i &=& \rx \tb_i + \tx \label{ax=xb.trans}
\end{eqnarray}
where the unknowns are now $\rx$, $\tx$ and $\lambda$.

%Notice that this formulation also contains the usual formulation where
%camera motions are obtained by pose computation. Indeed, setting
%$\lambda=1$ and $\ua{}_i=\ta_i$ brings back onto the latter.

\subsection{Linear formulation}
%The classical formulations are based on the axis/angle representation
%of a rotation, either implicitly~\citep{horaud95e,daniilidis96b} or
%explicitly~\citep{tsai89}. However, such reduced representations are
%ill-conditioned when the angle of rotations tend to zero, i.e. in the case
%of small rotations of the hand-eye device. 

We propose a new formulation which handles rotations of any kind.
Its underlying idea is to embed the rotation part of the problem,
intrinsically lying in $SO(3)$, in a larger space in order to
deliberately free ourselves from the non-linear orthogonality
constraint. This allows us to easily find 
a subspace of matrices verifying~(\ref{ax=xb.eq}). Then, the
application of the orthogonality constraint selects, in this subspace,
the unique rotation which is solution to the problem. 
This general idea is very powerful here since, as we will see, the
non-linear orthogonality constraint reduces to a linear norm
constraint. 

\begin{figure}[h!]
\psfrag{X}{$\matx{X}$}
\psfrag{a1}{$\matx{A}_i(\lambda)$}
\psfrag{B1}{$\matx{B}_i$}
\psfrag{R1}{ }
\psfrag{R2}{ }
\psfrag{Rn}{ }
\centerline{\resizebox{!}{6.5cm}{\includegraphics{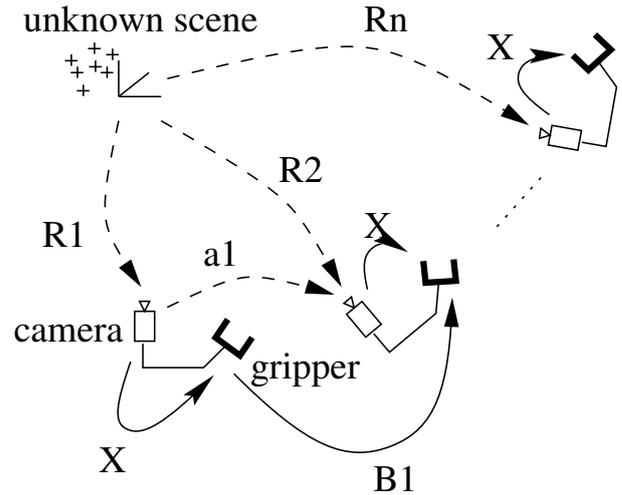}}}
\caption{From images of an unknown scene and the
knowledge of the intrinsic parameters of the camera,
structure-from-motion algorithms estimate, up to an unknown scale
factor $\lambda$, the camera motions $\matx{a}_i(\lambda)$.}
\label{rec3d} 
\end{figure}

The new formulation is inspired by the similarity
of~(\ref{ax=xb.rot}) with the Sylvester equation: $\matx{U} \matx{V} + 
\matx{V} \matx{W} = \matx{T}$. This matrix equation, which often
occurs in system theory~\citep{Brewer78}, is usually formulated as a
linear system~\citep{Rotella89,Hu92,Deif95}: 
\begin{equation}
 (\matx{U} \otimes \matx{I} + \matx{I} \otimes \matx{W}) vec(\matx{V})
= vec(\matx{T})
\label{sylvester}
\end{equation}
One fundamental property of the Kronecker product is~\citep{Brewer78}:
\begin{equation}
vec(\matx{C}\matx{D}\matx{E})=(\matx{C}\otimes\matx{E}^T)vec(\matx{D})
\label{fond}
\end{equation}
where $\matx{C}$,$\matx{D}$,$\matx{E}$ are any matrices with adequate
dimensions. Applying this relation to equation~(\ref{ax=xb.rot})
yields:
\begin{equation}
(\ra_i \otimes \rb_i) vec(\rx) = vec(\rx)
\label{kron.rot}
\end{equation}
Introducing the notation $vec(\rx)$ in
equation~(\ref{ax=xb.trans}), we obtain:
\begin{equation}
(\id{3} \otimes (\tb_i^T)) vec(\rx) + (\id{3} - \ra_i) \tx - \lambda
\ua{}_i =0
\label{kron.trans}
\end{equation}
We can then state the whole problem as a single homogeneous linear system:
\begin{multline}
\label{syst}
\left(
\begin{matrix}
\id{9} -\ra_i \otimes \rb_i &
\zero{9}{3}&\zero{9}{1}\\
\id{3} \otimes (\tb_i^T) & \id{3} - \ra_i&-\ua{}_i
\end{matrix}
\right) \\
 \times
\left(
\begin{matrix}
vec(\rx)\\
\tx\\
\lambda
\end{matrix}
\right) =
\left(
\begin{matrix}
\zero{9}{1}\\
\zero{3}{1}
\end{matrix}
\right)
\end{multline}

The question is now: ``What is the condition for this system to have a 
unique solution?'' and a subsequent one is: ``What occurs when this
condition is not fulfilled?''

\section{Algebraic analysis}
\label{analysis}
From earlier work on hand-eye calibration~\citep{tsai89,chen91c}, we
know that two motions with non-parallel rotation axes are sufficient to
determine the hand-eye transformation. 
We will show in this section, that our new linear solution owns the 
same sufficient condition but also allows us to identify
what can be obtained when such a sufficient condition is not fulfilled
(the so-called {\em partial calibration}).

Hence, let us determine what can be obtained using various
combinations of end-effector motions by successively considering:
pure translations, pure rotations, planar motions (i.e. containing the
same rotational axis and independent translations) and finally
general motions.
The results of this study are gathered up in Table~\ref{resume}.
Notice that by inverting the roles of the end-effector and the camera, we
obtain the same results for the recovery of the {\em eye-hand}
transformation (i.e. the inverse of the hand-eye transformation).

\begin{table*}[t!]
\caption{\label{resume}  Summary of the results for two independent motions.}
\centerline{
%\begin{footnotesize}
\begin{tabular}{|c|c|c|c|}
\hline
%\parbox[b]{\largeurcolonne}
%{\hfill First\\  Second \hfill motion\\motion}&
\begin{tabular}{lr}
&Motion 1\\
&\\
Motion 2&
\end{tabular}
&
%\parbox[b]{\largeurcolonne}{ \begin{center} 
\begin{tabular}{c}
Translation\\$\matx{R}_B = \matx{I}$\\$\vect{t}_B \neq 0$\\
\end{tabular}
%\end{center}} 
& %\parbox[b]{\largeurcolonne}{ \begin{center}
\begin{tabular}{c}
Rotation\\ $\matx{R}_B \neq \matx{I}$\\ $\vect{t}_B = 0$\\
\end{tabular}
%\end{center}}
&% \parbox[b]{\largeurcolonne}{ \begin{center}
\begin{tabular}{c}
General motion \\ $\matx{R}_B \neq \matx{I}$\\ $\vect{t}_B \neq 0$\\
\end{tabular}\\
%\end{center}}\\
%\parbox[b]{\largeurcolonne}{motion}&&&\\
\hline
%\parbox[c]{\largeurcolonne}{ \begin{center}
\begin{tabular}{c}
Translation\\ $\matx{R}_B = \matx{I}$\\ $\vect{t}_B \neq 0$
\end{tabular}
%\end{center}}
& %\parbox[c]{\largeurcolonne}{ \begin{center}
\begin{tabular}{c}
%$\emptyset$\footnote{If a third independent translation is available,
%then $\rx$ can be estimated.}
$\rx,\lambda$
\end{tabular}
%\end{center}}
&% \parbox[c]{\largeurcolonne}{ \begin{center}
\begin{tabular}{c}
$\rx,\lambda$\\$\tx(\alpha)$
\end{tabular}
%\end{center}}
&% \parbox[c]{\largeurcolonne}{ \begin{center}
\begin{tabular}{c}
$\rx,\lambda$\\$\tx(\alpha)$
\end{tabular}\\
%\end{center}}\\
\hline
%\parbox[t]{\largeurcolonne}{ \begin{center}
\begin{tabular}{c}
Rotation\\ $\matx{R}_B \neq \matx{I}$\\ $\vect{t}_B = 0$
\end{tabular}
%\end{center}}
& %\parbox[b]{\largeurcolonne}{ \begin{center}
\begin{tabular}{c}
$\rx,\lambda$\\$\tx(\alpha)$
\end{tabular}
%\end{center}}
&% \parbox[b]{\largeurcolonne}{ \begin{center}
\begin{tabular}{c}
$\rx, \tx(\lambda)$\\ \tiny Decoupled\\ \tiny solution
\end{tabular}
%\end{center}}
&% \parbox[b]{\largeurcolonne}{ \begin{center}
\begin{tabular}{c}
$\rx, \tx, \lambda$\\ \tiny General \\ \tiny solution
\end{tabular}\\
%\end{center}}\\
\hline
%\parbox[t]{\largeurcolonne}{ \begin{center}
\begin{tabular}{c}
General motion\\ $\matx{R}_B \neq \matx{I}$\\ $\vect{t}_B \neq 0$
\end{tabular}
%\end{center}}
&% \parbox[b]{\largeurcolonne}{ \begin{center}
\begin{tabular}{c}
$\rx, \lambda$\\$\tx(\alpha)$
\end{tabular}
%\end{center}}
&% \parbox[b]{\largeurcolonne}{ \begin{center}
\begin{tabular}{c}
$\rx, \tx, \lambda$\\ \tiny General\\ \tiny solution
\end{tabular}
%\end{center}}
&% \parbox[b]{\largeurcolonne}{ \begin{center}
\begin{tabular}{c}
$\rx, \tx, \lambda$\\ \tiny General\\ \tiny solution
\end{tabular}\\
%\end{center}}\\
\hline
\end{tabular}
%\end{footnotesize}
}
\end{table*}

\subsection{Pure translations} 

Recall from equation~(\ref{tsai1}), that when end-effector
motions are pure translations (i.e. $\rb_i=\id{3}$), then camera
motions are pure translations too (i.e. $\ra_i=\id{3}$). 
Hence, equation~(\ref{tsai2}) becomes
\begin{equation}
\ta_i = \rx \tb_i
\label{ta=rtb}
\end{equation}
Consequently, the amplitude of camera motion is the same as the
amplitude of end-effector motion, which is not the case when rotations 
are involved. One can therefore keep control of the camera
displacements and guarantee that a small end-effector motion will not
generate an unexpected large camera motion.
Concerning calibration, we have the following result:

\begin{proposition}
\label{prop_trans}
Three independent pure translations yield a linear estimation of
hand-eye rotation $\rx$ and of the unknown
scale factor $\lambda$. Hand-eye translation can not be observed.
\end{proposition}

\begin{proof}
In the case of pure translations, the upper part of the
system in~(\ref{syst}) vanishes and its lower part
simplifies into:
\begin{equation}
\left( \id{3} \otimes (\tb_i^T) \right) vec({\rx}) = \lambda \ua_i
\end{equation}

This implies that hand-eye translation $\tx$ can not be estimated.
However, the nine coefficients of the hand-eye rotation $\rx$ can be
obtained as we show below. This was also demonstrated
in~\citep{zhuang98b} in the particular case where $\lambda$ is known.

Let us assume temporarily that $\lambda$ is known.
If $\tb_i \neq 0$, then $\id{3} \otimes (\tb_i^T)$ has rank~3 since
\begin{equation} 
\id{3} \otimes (\tb_i^T) = \left(
\begin{matrix}
\tb_i^T& \zero{1}{3} & \zero{1}{3}\\ 
 \zero{1}{3}& \tb_i^T& \zero{1}{3}\\
\zero{1}{3} & \zero{1}{3} & \tb_i^T
\end{matrix} \right)
\end{equation}
Consequently, three linearly independent pure translations yield a full rank
$(9 \times 9)$ system:
\begin{equation}
\underbrace{\left(
\begin{matrix}
\id{3} \otimes (\tb_1^T)\\
\id{3} \otimes (\tb_2^T)\\
\id{3} \otimes (\tb_3^T)
\end{matrix}
\right)}_{\matx{M}}
vec(\rx)=
\lambda \left(
\begin{matrix}
\ua_1\\ \ua_2\\ \ua_3
\end{matrix}
\right)
\end{equation}
of which the solution $\trx$ is such that
\begin{equation}
vec(\trx)=\lambda \matx{M}^{-1}  \left(
\begin{matrix}
\ua_1\\ \ua_2\\ \ua_3
\end{matrix}
\right)
\label{sol_trans}
\end{equation}
Since
$(\matx{A}\otimes\matx{B})(\matx{C}\otimes\matx{D})=\matx{AC}\otimes\matx{BD}$~\citep{Bellman60}, 
it is easy to verify that the analytic form of the inverse of
$\matx{M}$ is:
\begin{multline}
\matx{M}^{-1} = \frac{1}{\Delta}
\big(
%\begin{matrix}
\id{3} \otimes (\tb_2 \times \tb_3) \\
\id{3} \otimes (\tb_3 \times \tb_1) \quad
\id{3} \otimes (\tb_1 \times \tb_2)
%\end{matrix}
\big)
\end{multline}
where $\times$ denotes the cross-product and
$\Delta=\det(\tb_1,\tb_2,\tb_3)$.
This allows the
rewriting of (\ref{sol_trans}) in closed form:
\begin{multline}
vec(\trx) = \frac{\lambda}{\Delta}
\big(
\id{3} \otimes (\tb_2 \times \tb_3) \ua_1 + \\
\id{3} \otimes (\tb_3 \times \tb_1) \ua_2 +
\id{3} \otimes (\tb_1 \times \tb_2) \ua_3 
\big)
\end{multline}
Applying (\ref{fond}) yields
\begin{multline}
vec(\trx) = \frac{\lambda}{\Delta}
vec \big(
\ua_1 (\tb_2 \times \tb_3)^T+ \\
\ua_2 (\tb_3 \times \tb_1)^T+
\ua_3 (\tb_1 \times \tb_2)^T
\big)
\end{multline}
and from the linearity of the $vec$ operator, we finally obtain:
\begin{multline}
\trx = \frac{\lambda}{\Delta}
\big(
\ua_1 (\tb_2 \times \tb_3)^T+ \\
\ua_2 (\tb_3 \times \tb_1)^T+
\ua_3 (\tb_1 \times \tb_2)^T
\big)
\label{rx_trans}
\end{multline}

Let us analyze this result and prove now that $\trx$ is
equal to $\rx$, when measurements are exact. To do that,
first recall that $\rx \tb_i = \lambda \ua_i$. Hence,
\begin{multline}
\label{nada}
\trx  = \frac{1}{\Delta} \rx \\
\times \underbrace{\left(
\tb_1 (\tb_2 \times \tb_3)^T+
\tb_2 (\tb_3 \times \tb_1)^T+
\tb_3 (\tb_1 \times \tb_2)^T
\right)}_{\matx{N}} 
\end{multline}
Recalling that $\rx$ is orthogonal and verifying that
$\matx{N}=\Delta \id{3}$, we obtain that $\trx = \rx$.

This analysis proves that even if $\lambda$ is unknown the columns of
$\trx$, estimated from~(\ref{rx_trans}), are orthogonal to each other.
Thus, only the unity constraint (i.e. $\det(\trx)=1$) remains to be
verified by $\trx$. From~(\ref{rx_trans}) again, the unity constraint
immediately gives $\lambda$. Consequently, the hand-eye rotation can
be recovered from three linearly independent translations.
\end{proof}

\begin{proposition}
\label{deux_trans}
A minimum of 2 linearly independent pure translations are intrinsically 
enough to estimate the hand-eye rotation $\rx$ and the unknown scale
factor $\lambda$.
\end{proposition}
\begin{proof}
The solution is not linear any more and comes in two steps.
\begin{enumerate}
\item Scale factor estimation

As $\rx$ is orthogonal, it preserves the norm. Hence, for each pure
translation $i$, we have:
\[
\|\rx\tb_i\|=\|\tb_i\| 
\] 
Applying~(\ref{ta=rtb}) and (\ref{ua}) on the left-hand side of this
expression gives for all $i$:
\[
\lambda \| \ua_i \| = \| \tb_i\|
\]
where $\ua_i$ and $\tb_i$ are known.

\item Hand-eye rotation estimation

Remark that if $\tb_1$ and $\tb_2$ are two linearly independent
vectors, then $\tb_1 \times \tb_2$ is linearly independent from them.
Moreover, one can prove that
\[
\rx (\tb_1 \times \tb_2) = (\rx \tb_1) \times (\rx \tb_2)
\]
Therefore, we can form the following full-rank $(9 \times 9)$ system:
\begin{multline}
\left(
\begin{matrix}
\id{3} \otimes (\tb_1^T)\\
\id{3} \otimes (\tb_2^T)\\
\id{3} \otimes ((\tb_1 \times \tb_2)^T)
\end{matrix}
\right) \\
\times vec(\rx)=
\lambda \left(
\begin{matrix}
\ua_1\\ \ua_2\\ \lambda (\ua_1 \times \ua_2)
\end{matrix}
\right)
\end{multline}
Since $\lambda$ is now known, $\rx$ can be obtained by inverting this
system and the orthogonality of the solution is guaranteed by the
proof of Proposition~\ref{prop_trans}.
\end{enumerate}
\end{proof}

\subsection{Pure rotations} 
By ``pure rotations'', we mean motions of the end-effector such that
$\tb_i = 0$. In practice, these motions can be realized by most of the 
robotic arms, since the latter are usually designed in such a manner that
their end-effector reference frame is centered on a wrist (i.e. the
intersection of the last three revolute joint axes). For
similar reasons, pan-tilt systems may also benefit from the subsequent
analysis.

In such a case, we can state the following proposition:
\begin{proposition}
\label{prop_rot}
If the robot end-effector undergoes $n \ge 2$ pure rotations, 2 of
which having non-parallel axes, then one can linearly estimate the
hand-eye rotation $\rx$ and the hand-eye translation up to the unknown
scale factor $\tx/\lambda$. These two estimations are decoupled.
\end{proposition}

Notice that, in the case where camera motion is obtained
through pose computation, $\lambda$ is known and the hand-eye
translation can thus be fully recovered, as does Li~\citep{li98}.

\begin{proof}
With pure rotations, the system in~(\ref{syst}) is block-diagonal and decouples into:
\begin{eqnarray}
\left( \id{9} -\ra_i \otimes \rb_i \right)vec(\rx) &=& \zero{9}{1}
\label{pour_svd}, i=1..n\\
\left(  \id{3} - \ra_i \right) \tx &=& \lambda \ua_i \label{pour_mc}, i=1..n
\end{eqnarray}

Let us first study eq.~(\ref{pour_mc}).
Consider $n \ge 2$ rotations, such that, without loss of generality,
$\ra_1$ and $\ra_2$ have non parallel axes and note $\tx=\lambda
\tx_0$.
Then,  from (\ref{pour_mc}), we form the following
system:
\begin{equation}
\begin{pmatrix}
\id{3} - \ra_1\\
\vdots\\
\id{3} - \ra_n\\
\end{pmatrix}
\lambda \tx_0
= 
\lambda 
\begin{pmatrix}
\ua_1\\ \vdots \\ \ua_n
\end{pmatrix}
\end{equation}
This system has a 1-dimensional solution subspace. Indeed,
$\tx_0$ is solution to the following system: 
\begin{equation}
\begin{pmatrix}
\id{3} - \ra_1\\
\vdots\\
\id{3} - \ra_n\\
\end{pmatrix}
\tx_0
= 
\begin{pmatrix}
\ua_1\\ \vdots \\ \ua_n
\end{pmatrix}
\end{equation}
which has full rank because $\ra_1$ and $\ra_2$ have non parallel
axes. 

Notice that if $n>2$, $\tx_0$ is the least squares solution of
this system.
Notice also that the parameter of the subspace is
the unknown scale factor. This is not surprising since pure rotations
of the robot do not contain metric information.

Let us now study the first subsystem (\ref{pour_svd}). One of the properties of the
Kronecker product is that the eigenvalues of $\matx{M} \otimes
\matx{N}$ are the product of the eigenvalues of $\matx{M}$ by those of
$\matx{N}$. In our case, $\ra_i$ and $\rb_i$ have the same
eigenvalues: \mbox{$\{ 1, e^{i \theta_i}, e^{-i \theta_i} \}$} and thus the
eigenvalues of \mbox{$\ra_i \otimes \rb_i$} are: $\{ 1,1,1,e^{i \theta_i},
e^{i \theta_i}, e^{-i \theta_i}, e^{-i \theta_i}, e^{2i \theta_i},
e^{-2i \theta_i}\}$. 

Consequently, when the angle of rotation
$\theta_i$ is not a multiple of $\pi$, then the $(9 \times 9)$ matrix
of~(\ref{pour_svd})
$\id{9} -\ra_i \otimes \rb_i$ has rank~6. Hence, the solution $\rx$ lies in a
3-dimensional manifold. Using the orthogonality constraints of $\rx$, the
solution manifold dimension can only be reduced to 1. This dimension
corresponds to the common eigenvalue 1 of $\ra_i$ and $\rb_1$. This confirms
the need for two rotations.

In the case of two or more independent rotations, we can state the
following lemma:
\begin{lemma}
\label{la_solution}
If the robot end-effector undergoes $n \ge 2$ pure rotations, 2 of
which having non parallel axes,
then system (\ref{pour_svd}) has rank 8, its null space $\mathcal{K}$ is
1-dimensional and the hand-eye rotation $\rx$ is equal to:
\begin{equation}
\rx =
\frac{sign(\det(\matx{V}))}{|\det(\matx{V})|^{\frac{1}{3}}}\matx{V}
\end{equation}
where $sign()$ returns the sign of its argument,
$\matx{V}=vec^{-1}(\vect{v})$ and $\vect{v}$ is any vector of the
null space $\mathcal{K}$.
\end{lemma}

The proof of this Lemma is technical and, hence, given in Appendix~\ref{preuve}.
In this proof, we determine the null-space of the matrix:
\[
\left(
\begin{matrix}
\id{9} -\ra_1 \otimes \rb_1\\
\id{9} -\ra_2 \otimes \rb_2\\
\vdots\\
\id{9} -\ra_n \otimes \rb_n
\end{matrix}
\right)
\]
where, without loss of generality, $\ra_1$ and $\ra_2$ are the two
rotations with non parallel axes.
We show that, in the noise-free case, any matrix $\matx{V}$ extracted 
for the null space $\mathcal{K}$ is guaranteed to have perpendicular
rows, even if more than the minimal 2 rotations with non parallel axes
are used.
\end{proof}

In practice, $\vect{v}$ can be determined using a Singular Value
Decomposition (SVD) which is known to accurately estimate the null space
of a linear mapping. Nevertheless, noise yields a non orthogonal
matrix, which requires to be corrected by an orthogonalization
step\citep{horn88a}. However, even though it has not been quantified, the results we
obtained before orthogonalization are close to be orthogonal.

\begin{figure}[t]
\centerline{\resizebox{!}{3.5cm}{\includegraphics{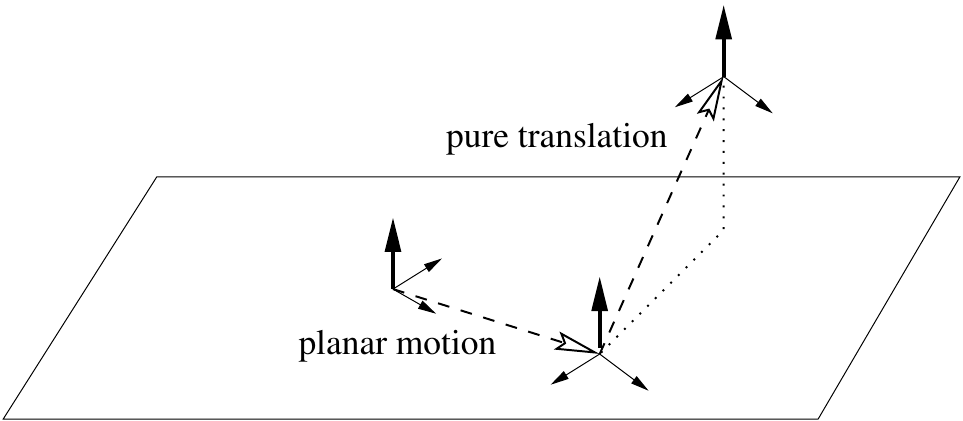}}}
\caption{One planar motion with non-identity rotation and one non-zero pure translation
which is not parallel to the rotation axis of the planar motion.}
\label{plan_trans} 
\end{figure}

\subsection{Planar motions} 
Some robots are restricted to move on a
plane, such as car-like robots.  In this case, all the robot and
camera rotations have the same axis $\vect{n}_b$
(resp. $\vect{n}_a=\rx \vect{n}_b$), which is orthogonal to the plane
of motion. Then, we can demonstrate that

\begin{lemma}
\label{plan+trans}
One planar motion with non-identity rotation and one non-zero pure translation
(which is not parallel to the rotation axis of the planar motion, see Fig.~\ref{plan_trans}) are
intrinsically enough to recover the hand-eye rotation $\rx$ and the unknown scale factor
$\lambda$. 
The hand-eye translation
can only be estimated up to an unknown height $\alpha$ along the
normal to the camera plane of motion (Fig.~\ref{car}).
\end{lemma}

\begin{figure}[t]
\centerline{\resizebox{!}{3.5cm}{\includegraphics{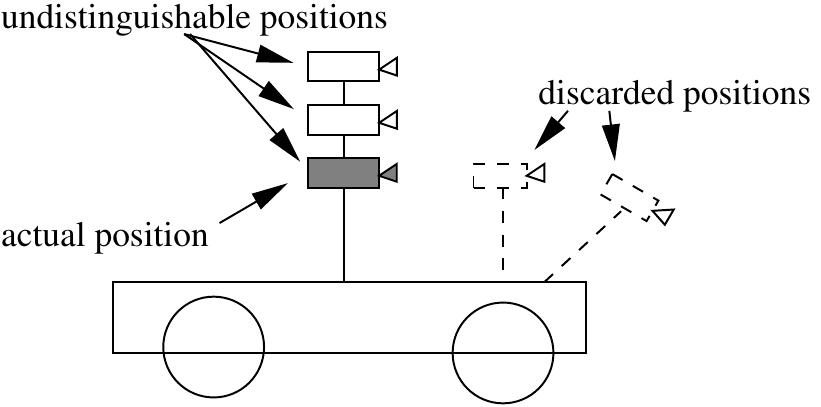}}}
\caption{In the case of planar motions, one can not
determine the altitude of a camera which is rigidly mounted onto the base.}
\label{car} 
\end{figure}

Notice that this Lemma is not limited to the planar motion case, since the
pure translation is not restricted to lie in the plane of motion.

\begin{proof}
Assume without loss of generality that the first motion is
a pure translation ($\ra_1=\rb_1=\id{3}$, $\tb_1 \in \Re^3$) and the
second is a planar motion with non-identity rotation
such that its rotation axis $\vect{n}_b$ is not
parallel to $\tb_1$ (Fig.~\ref{plan_trans}).  
Then, the general system~(\ref{syst}) rewrites as:
\begin{align}
\label{plan+trans:eq}
\left(
\begin{matrix}
\id{3} \otimes (\tb_1^T) & \zero{3}{3} &-\ua_1\\
\id{9} -\ra_2 \otimes \rb_2 &
\zero{9}{3}&\zero{9}{1}\\
\id{3} \otimes (\tb_2^T) & \id{3} - \ra_2&-\ua_2
\end{matrix}
\right) 
\\ \times
\left(
\begin{matrix}
vec(\rx)\\
\tx\\ \lambda
\end{matrix}
\right) =
\zero{15}{1}
\end{align}
which is equivalent to the following two equations
\begin{align}
\left(
\begin{matrix}
\id{9} -\ra_2 \otimes \rb_2  \zero{9}{1}\\
\id{3} \otimes (\tb_1^T) -\ua_1
\end{matrix}
\right) 
\left(
\begin{matrix}
vec(\rx)\\ \lambda
\end{matrix}
\right)  =
\zero{12}{1} \label{dix-sept}\\
\label{dix-huit}
( \id{3} - \ra_2) \tx  =-\lambda\ua_2
%\nonumber
-\left(\id{3} \otimes (\tb_2^T)\right)vec(\rx)
\end{align}

The solution comes in three steps:

\begin{enumerate}
\item Scale factor estimation.
As in the proof of Proposition~\ref{deux_trans}.

\item Hand-eye rotation estimation.
Recall that the camera axis of rotation $\vect{n}_a$ and the robot axis of
rotation $\vect{n}_b$ are related by:
\[
\rx \vect{n}_b = \vect{n}_a
\]
which is similar to (\ref{ta=rtb}). 
%
%Hence, $\vect{n}_b$ can be considered as a 
%virtual pure translation of the robot. 
%
Since $\tb_1$ and $\vect{n}_b$ are assumed to be non-parallel, they are
linearly independent. Therefore, we obtain, as in
the proof of Proposition~\ref{deux_trans}, a full-rank $(9 \times 9)$
system where $\rx$ is the only unknown:
\begin{multline}
\left(
\begin{matrix}
\id{3} \otimes (\tb_1^T)\\
\id{3} \otimes (\vect{n}_b^T)\\
\id{3} \otimes ((\tb_1 \times \vect{n}_b)^T)
\end{matrix}
\right)
vec(\rx)
\\ =
\left(
\begin{matrix}
\lambda \ua_1\\ \vect{n}_a\\ \lambda (\ua_1 \times \vect{n}_a)
\end{matrix}
\right)
\end{multline}

\item Hand-eye translation estimation.
We can insert the estimated $\rx$ and $\lambda$ into~(\ref{dix-huit}) and
obtain a system, where only $\tx$ is unknown. This system is always
under-constrained. Hence, it admits as solution 
any vector of the form 
\begin{equation}
\tx(\alpha) = \vect{t}_\perp + \alpha \vect{n}_a
\end{equation}
where $\alpha$ is any scalar value and $\vect{t}_\perp$ is a solution
in the plane of the camera motion. The latter vector is unique since
$\id{3}-\ra_1$ has rank~2 and the plane of motion is
\mbox{2-dimensional}. In practice, $\vect{t}_\perp$ can be obtained by an
SVD of $\id{3}-\ra_1$~\cite[\S2.6]{press92a}.   
\end{enumerate}
\end{proof}

The previous Lemma serves as a basis to the case of planar motions as:

\begin{proposition}
Two planar motions allow the
estimation of the hand-eye rotation $\rx$  
and the unknown scale factor $\lambda$ if one the following three
sets of conditions is fulfilled:
\begin{itemize}
\item the two motions are linearly independent pure translations
\item one of the two motions is a non-zero pure translation
\item the two motions contain a non-identity rotation and 
\[ (\id{3} - \rb_2)\tb_1-(\id{3}-\rb_1)\tb_2 \neq 0 \]
\end{itemize}
In the last two cases, the hand-eye translation can only be estimated
up to an unknown height $\alpha$ along the normal to the camera plane
of motion (Fig.~\ref{car}). 
\end{proposition}

\begin{proof}
The first set of conditions falls back into the pure translation case
and Proposition~\ref{deux_trans} apply. The second set of conditions
is contained in Lemma~\ref{plan+trans}.

Let us now show that the last set of conditions can be brought back to 
the second one. To do that, consider the system which is built upon the
two planar motions:
\begin{multline}
\label{two_planar}
\begin{matrix}
L_1 \rightarrow\\
L_2 \rightarrow\\
L_3 \rightarrow\\
L_4 \rightarrow
\end{matrix}
\left(
\begin{matrix}
\id{9} -\ra_1 \otimes \rb_1 &
\zero{9}{3}&\zero{9}{1}\\
\id{3} \otimes (\tb_1^T) & \id{3} - \ra_1&-\ua_1\\
\id{9} -\ra_2 \otimes \rb_2 &
\zero{9}{3}&\zero{9}{1}\\
\id{3} \otimes (\tb_2^T) & \id{3} - \ra_2&-\ua_2
\end{matrix}
\right)
\\ \times
\left(
\begin{matrix}
vec(\rx)\\
\tx\\ \lambda
\end{matrix}
\right)  =
\zero{15}{1}
\end{multline}
The block line $L_1$ and the third one $L_3$ of this system are equivalent since 
both motions have the same rotation axis. Hence, we can discard the
first one and obtain:
\begin{multline}
\begin{matrix}
L'_1 \rightarrow\\
L'_2 \rightarrow\\
L'_3 \rightarrow
\end{matrix}
\left(
\begin{matrix}
\id{3} \otimes (\tb_1^T) & \id{3} - \ra_1&-\ua_1\\
\id{9} -\ra_2 \otimes \rb_2 &
\zero{9}{3}&\zero{9}{1}\\
\id{3} \otimes (\tb_2^T) & \id{3} - \ra_2&-\ua_2
\end{matrix}
\right)
\\ \times
\left(
\begin{matrix}
vec(\rx)\\
\tx\\ \lambda
\end{matrix}
\right) =
\zero{15}{1}
\end{multline}
Consider now the linear combination 
$(\id{3} -\ra_2)L'_1-(\id{3}-\ra_1)L'_3$ which gives:
\begin{align}
\big[ & (\id{3} -\ra_2) \big(\id{3} \otimes (\tb_1^T)\big) \nonumber \\
& -(\id{3}-\ra_1)\big(\id{3} \otimes (\tb_2^T)\big) \big] vec(\rx)
\label{ligne1}\\
& + \big[ (\id{3} -\ra_2)(\id{3}-\ra_1)-(\id{3}-\ra_1)(\id{3} -\ra_2)
\big] \tx 
\label{ligne2}\\
& - \lambda \big[(\id{3} -\ra_2)\ua_1-(\id{3}-\ra_1)\ua_2 \big] =0
\label{ligne3}
\end{align}
As $\ra_1$ and $\ra_2$ have the same rotation axis, they commute and,
hence,  $(\id{3} -\ra_2)(\id{3}-\ra_1)-(\id{3}-\ra_1)(\id{3}
-\ra_2)=0$. Therefore, the term on line (\ref{ligne2}) is null.
As for the term on line~(\ref{ligne3}), let us denote it as
$\ua'_1$.

Let us now consider the first term~(\ref{ligne1}) and show that it can be
rewritten under the form $\rx\tb_1^{'T}$. To do that, recall
that $(\id{3} \otimes \tb_i^T) vec(\rx)=\rx \tb_i$.
Hence, the first term equals:
\[
(\id{3} -\ra_2)\rx \tb_1 - (\id{3}-\ra_1)\rx\tb_2
\]
Using $\ra_i \rx = \rx \rb_i$, we then obtain:
\[
\rx \big( \underbrace{(\id{3}-\rb_2) \tb_1 - (\id{3}-\rb_1)\tb_2}_{\tb'_1} \big)
\]
%which is equal to:
%\[
%\big[ \id{3} \otimes  \big( \underbrace{(\id{3}-\rb_2) \tb_1 -
%(\id{3}-\rb_1)\tb_2}_{\tb'_1} \big)^T \big] vec(\rx)
%\]
Consequently, $L'_1$ is equivalent to:
\[
\rx \tb'_1 = \lambda \ua'_1
%\left(
%\begin{matrix}
%\id{3}\otimes(\tb_1^{\prime T})&\zero{3}{3}&-\ua'_1
%\end{matrix}
%\right)
%\left(
%\begin{matrix}
%vec(\rx)\\
%\tx\\ \lambda
%\end{matrix}
%\right) =
%\zero{3}{1}
\]
where we recognize the pure translation case. Hence,
system~(\ref{two_planar}) rewrites under the same form as
in~(\ref{plan+trans:eq}) of Lemma~\ref{plan+trans}. Therefore, a
solution exists  if the virtual robot pure translation $\tb'_1$ is not
parallel to $\vect{n}_b$. As both $\tb_1$ and $\tb_2$ are orthogonal to
$\vect{n}_b$, this condition reduces to a non zero condition on
$\tb'_1$, which is expressed as:
\[ (\id{3}-\rb_2) \tb_1 - (\id{3}-\rb_1)\tb_2 \neq 0 \]

\end{proof}

In conclusion, we exhibited sufficient conditions to obtain,
from two planar motions, the hand-eye rotation and the hand-eye
translation, up to a component perpendicular to the camera plane of
motion. In the case of a car, this unknown component can be
interpreted as a height with respect to the base of the
car~(Fig.~\ref{car}).

\subsection{The general case}
In the case of two independent general motions with non-parallel axes,
there exists a unique solution to the hand-eye calibration problem.
We obtain the same result for our hand-eye self-calibration problem:

\begin{proposition}
If the robot end-effector undergoes two independent general motions
with non-parallel axes, then the hand-eye transformation $(\rx,\tx)$
can be fully recovered, as well as the Euclidean reconstruction
unknown scale factor $\lambda$.
\end{proposition}

Using our formulation, one possibility to solve the whole system
in~(\ref{syst}) is to find its null space, which is a subspace of
$\Re^{13}$.  The latter subspace must be 1-dimensional and only depend on
$\lambda$, according to the sufficient condition for hand-eye
calibration. Hence, the solution to hand-eye self-calibration is a
$13 \times 1$ vector to be found in a 1-dimensional subspace. It can
therefore be extracted from this null space by applying the unity
constraint to the first 9 coefficients representing the hand-eye
rotation, as seen in the pure translation case. 

However, Wei {\em et al}~\citep{wei98a} remarked, in the case where
camera motions are obtained through pose computation, that the
accuracy of the simultaneous estimation of hand-eye rotation and
translation is not independent of the physical unit used for the
translation. By analogy with this remark, solving directly for the
whole system may yield the same dependence.
In addition,
such a solution does not guarantee that the estimated $\rx$
is an orthogonal matrix. Then, one has to perform a correction
of the result by applying the orthogonality constraint. However, this
correction is non-linear in essence and it is hence improbable to find
the corresponding correction on the hand-eye translation estimation.

On the opposite, a two-step solution, as in~\citep{tsai89},
guarantees an orthogonal estimate of the hand-eye rotation. Indeed,
the first step consists in the linear estimation of the hand-eye
rotation as in the case of pure rotations~(\ref{pour_svd}), which had
this property. Recall that the solution is given by estimating the
nullspace of
\begin{equation}
\left(
\begin{matrix}
\id{3}-\ra_1 \otimes \rb_1\\
\vdots\\
\id{3}-\ra_n \otimes \rb_n
\end{matrix}
\right) vec(\rx) = 0
\end{equation}
and, then, by proceeding to an orthogonalization step to
counterbalance the effect of noise. 

This part of the
solution is, in essence, not very different from the existing
approaches but, here again, we want to point out that our goal 
is to give a generic formulation in which the particular cases
and the general case can be unified.

As for the second step, it exploits the remaining lines
in~(\ref{syst}):
\begin{equation}
\left(
\begin{matrix}
\id{3}-\ra_1& -\ua_1\\
%\id{3}-\ra_2& -\ua_2\\
\vdots\\
\id{3}-\ra_n& -\ua_n
\end{matrix}
\right) 
\left(
\begin{matrix}
\tx \\ \lambda
\end{matrix}
\right) 
 = \left(
\begin{matrix}
- \rx \tb_1\\% -\rx \tb_2\\
\vdots \\ - \rx \tb_n
\end{matrix}
\right) 
\end{equation}

As soon as two general motions with non parallel axes are used, this system
has full rank. It thus yields a linear least squares solution to the hand-eye
translation and the scale factor.

\section{Experiments}
\label{expe}

In this section, we will first choose a distance to measure the errors
between rigid transformations since their group $SE(3)$ does not hold
an intrinsic metric~\citep{murray94a}. Second, we will show some
simulation results to test the robustness to noise of our method,
compared to the reference methods. Finally, we will give experimental
results in real conditions. Notice that more experimental results can 
be found in~\citep{mathese}.

In this section, we numbered the methods we compared as follows:
axis/angle method~\citep{tsai89} ({\bf M1}), dual quaternion
method~\citep{daniilidis96b} ({\bf M2}),  non-linear
minimization~\citep{horaud95e} ({\bf M3}), our linear formulation
adapted to the case where camera motions are obtained through pose
computation ({\bf M4}), and self-calibration ({\bf M5}).

\subsection{Error measurement}
To measure the errors in translation, we chose the usual relative
error in~$\Re^3$: \mbox{$\| \hat{\vect{t}}-\vect{t} \|/\| \vect{t} \|$}, where
the '$\hat{\ }$' notation represents the estimated value.

For the errors in orientation, no canonical measure is defined. We
chose the quaternion norm used in~\citep{daniilidis96b}: $\|
\hat{\quat{q}} - \quat{q} \|$ for its simplicity and its direct
relation to $\alpha$, the angle of the residual
rotation between these two orientations.
Indeed, if $\hat{\quat{q}}$ and
$\quat{q}$ are unitary, then $\|\hat{\quat{q}} - \quat{q} \| = 2
- 2 \cos \frac{\alpha}{2}$. It is thus strictly increasing from 0 to 4 
as $\alpha$ goes from 0 to $2 \pi$. Moreover, this metric avoids the
singularity in $\alpha=\pi$ appearing when using
geodesics~\cite[p.35]{samson91}. 

\subsection{Simulations}
We first performed simulations to gain some insight of the numerical
behavior of our linear method ({\bf M4})
with comparison to the reference methods ({\bf M1--M3}). We thus tested the
robustness of the methods to noise and their accuracy with respect to
the number of calibration motions in use.

\subsubsection{Simulation procedure}
For each simulation series and for each value of the parameter of
interest (noise, number of motions), we followed the same
methodology. First, we defined a hand-eye transformation by
random choice of the Roll-Pitch-Yaw angles of its rotation matrix as
well as of the coefficients of its translation vector, according to
Gaussian laws. Second, we similarly chose a sequence of robot motions
and defined, from it and the hand-eye transformation, the
corresponding camera motion sequence. Third, we added noise to the
camera motions (see below). Finally, we performed hand-eye calibration with the
various methods and compared their results to the initial hand-eye
transformation.

\subsubsection{Inserting noise}
We added noise to the camera translations $\vect{t}_{A_i}$ by defining 
$\tilde{\vect{t}}_{A_i}=\vect{t}_{A_i} + \nu \|\vect{t}_{A_i}\|
\vect{n}$ where $\nu$ is a scalar and $\vect{n}$ is a Gaussian 
3-vector with zero mean and unit variance (white noise).
As for the camera rotations, we added noise to their Roll-Pitch-Yaw
angles as \mbox{$\tilde{\alpha}= (1 + \nu r) \alpha$} where $\alpha$ is any
of these angles, $\nu$ is the same as for the translation and $r$ is a 
1-dimensional white-noise. Hence, $\nu$ defines a signal-to-noise ratio.

\subsubsection{Robustness to noise}
We tested for the value of $\nu$, making it vary
from 0 to 20\%
in two simulation series. In the first one,
we made 100 different choices of hand-eye transformations and motion
sequences for each noise level. These sequences contained only two
motions, with maximal amplitude of 1 m in translation and 180 $\deg$
in rotation.
Fig.~\ref{cmp_toutes_bruit} gathers the calibration errors. It shows 
that Tsai and Lenz's method ({\bf M1}) and ours ({\bf M4}) obtain the
highest accuracy in rotation. For translations, they are very powerful
as long as the noise level is low but are less accurate than the dual
quaternion method ({\bf M2}) or the non linear minimization method
({\bf M3}) when the noise level increases.

%It is interesting to 
%have a look a the CPU time required by each method
%(Fig.~\ref{cmp_toutes_bruit_tps}). 

\begin{figure*}[t!h!]
\centerline{
\begin{tabular}{cc}
\resizebox{!}{5cm}{\includegraphics{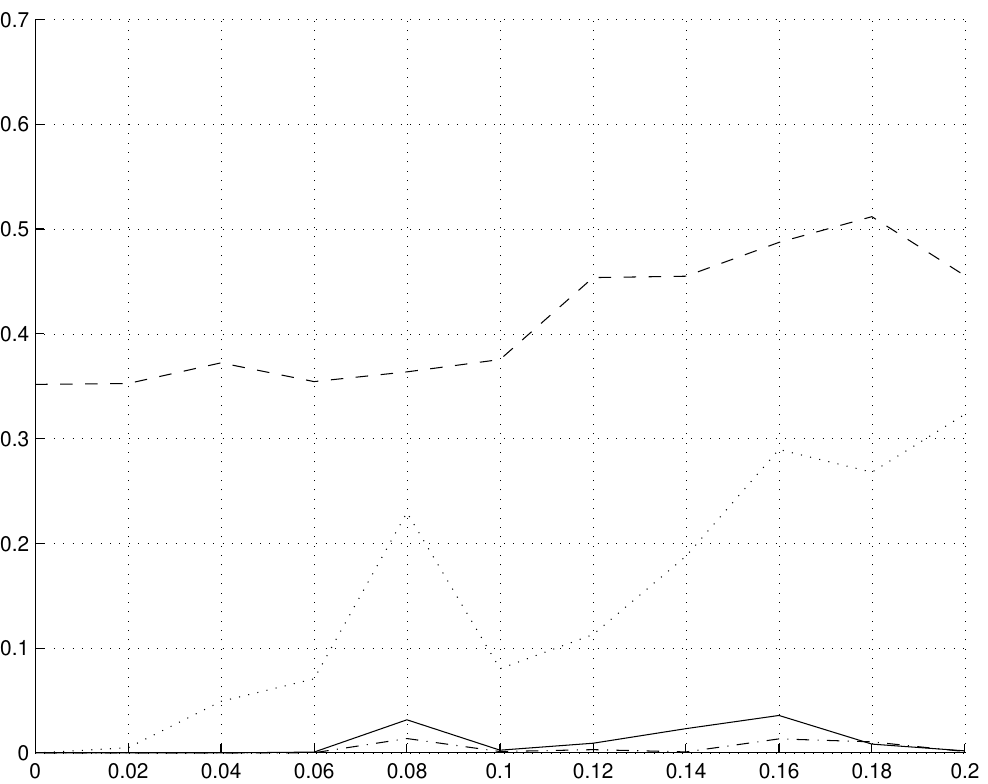}} &
\resizebox{!}{5cm}{\includegraphics{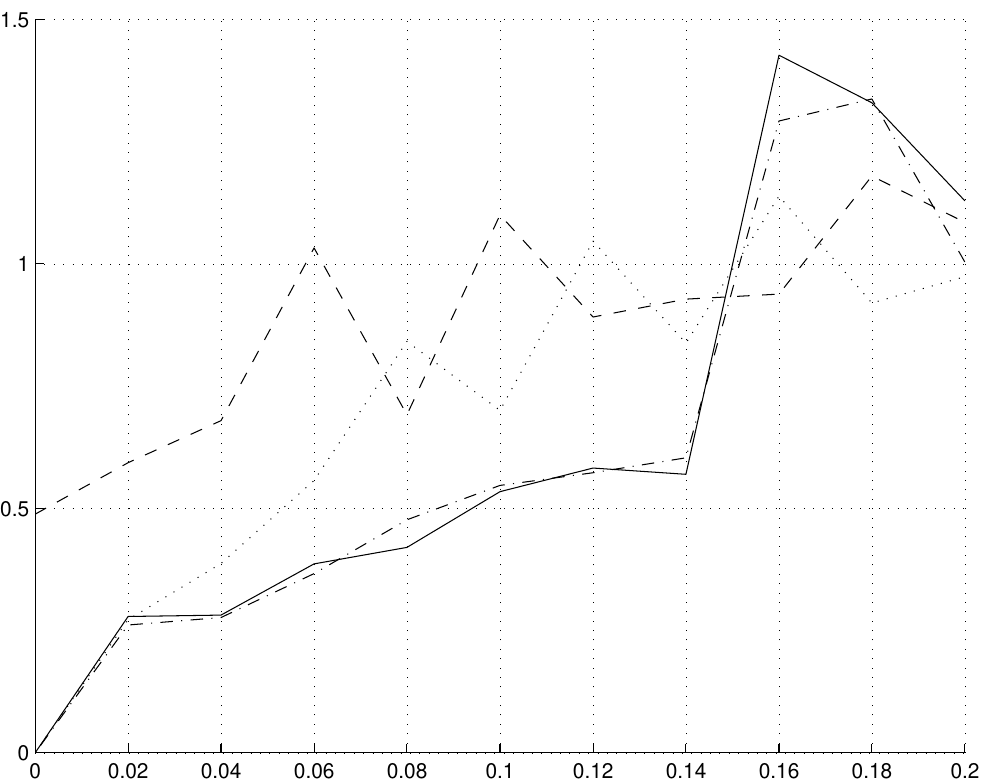}}
\end{tabular}
}
\caption{Rotation (left) and translation
(right) relative calibration errors with respect to noise level:
{\bf M1} (---),
{\bf M2} ($\cdots$), {\bf M3} 
(- -), {\bf M4} ($-\ \cdot$).} 
\label{cmp_toutes_bruit}
\end{figure*}

%\begin{figure}[btp]
%\centerline{\resizebox{!}{5cm}{\includegraphics{FIGURES/2_Grands_Mvts/cmp_lin_optim_bruit_tps-eps-converted-to.pdf}}}
%\caption{\label{cmp_toutes_bruit_tps}Mean CPU time (in seconds)
%associated to the results in Fig.~\ref{cmp_toutes_bruit}.}
%\end{figure}

In a second simulation series, we almost repeated the first one, just
reducing the amplitude of the calibration motions to 2 cm in
translation and 10 $\deg$ in rotation. The results
(Fig.~\ref{cmp_toutes_bruit_petits_mvts}) show that our linear
formulation is less sensitive to this reduction than the other
methods.

\begin{figure*}[t!h!]
\centerline{
\begin{tabular}{cc}
\resizebox{!}{5cm}{\includegraphics{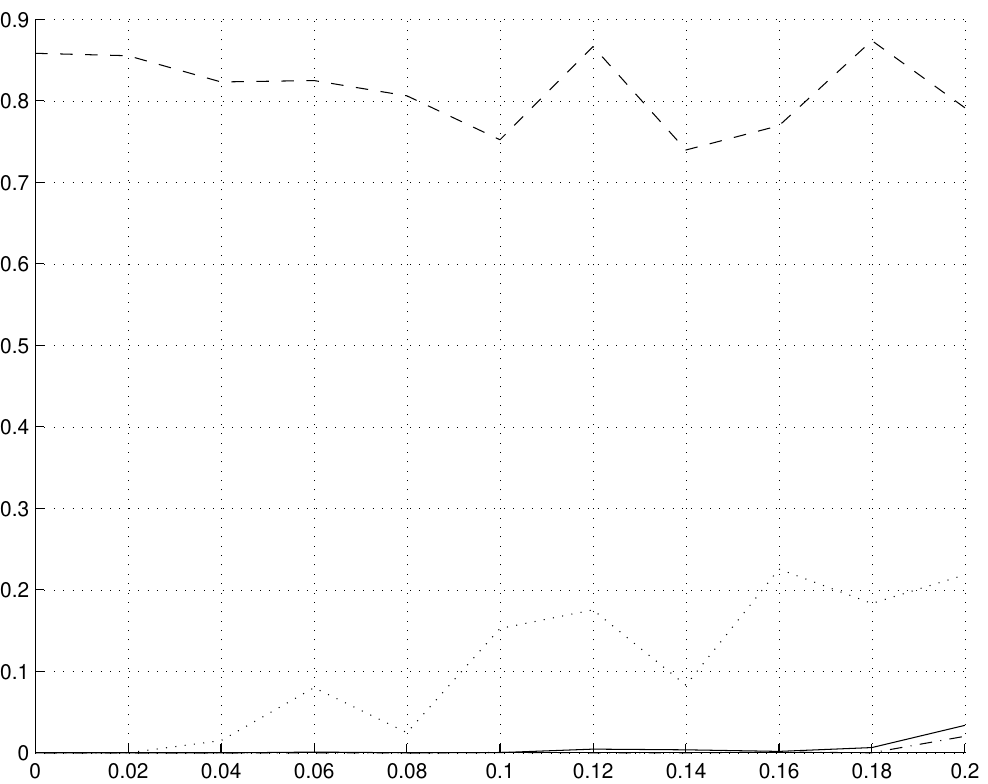}} &
\resizebox{!}{5cm}{\includegraphics{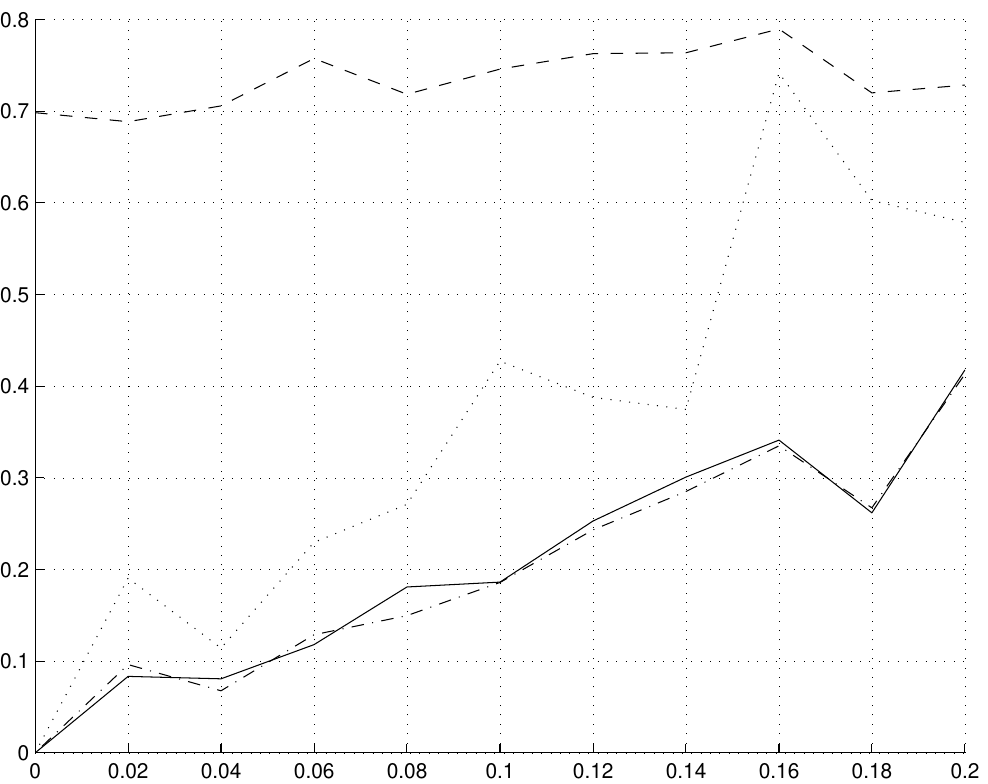}}
\end{tabular}
}
\caption{Calibration errors with
respect to noise level using
small motions (Same conventions as in Fig.~\ref{cmp_toutes_bruit})} 
\label{cmp_toutes_bruit_petits_mvts}
\end{figure*}

\subsubsection{Influence of motion number}

In this experiment, we kept the noise level constant ($\nu = 0.01$)
and generated sequences of varying length, i.e. from 2 to 15
calibration motions. Their amplitude was chosen to be
small (1 cm in translation and 10 $\deg$ in rotation).
For each sequence length, we proceeded to 100 random choices of
hand-eye transformations and calibration motions.
The results (Fig.~\ref{cmp_mvts}) show here again a higher accuracy
for our linear formulation.

\begin{figure*}[t!b!h!]
\centerline{
\begin{tabular}{cc}
\resizebox{!}{5cm}{\includegraphics{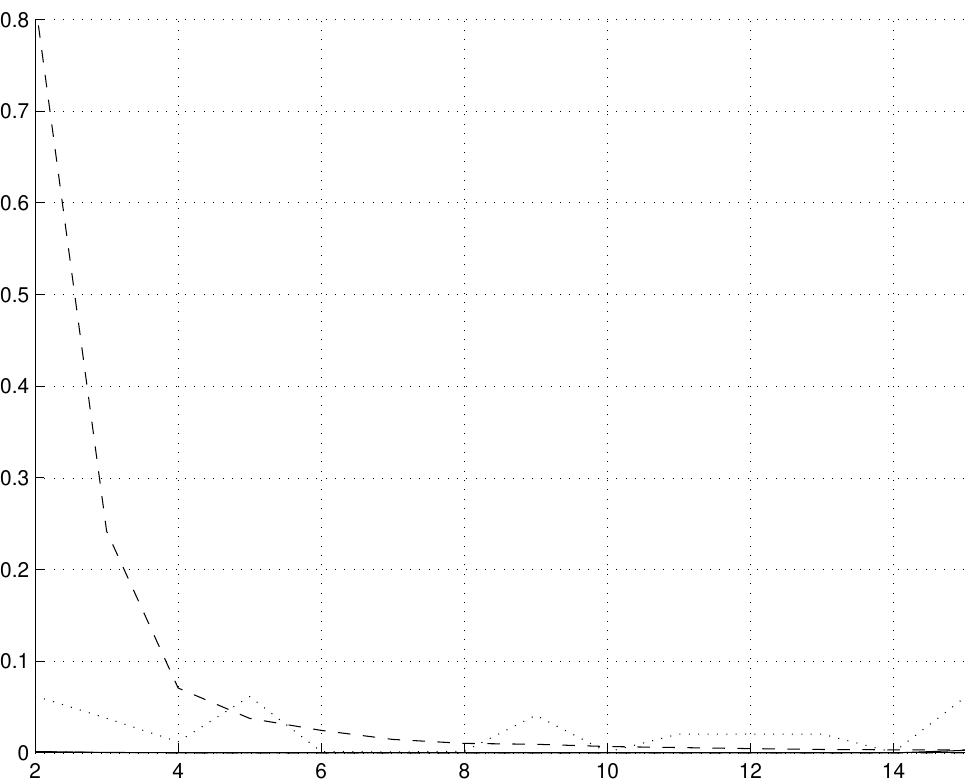}} &
\resizebox{!}{5cm}{\includegraphics{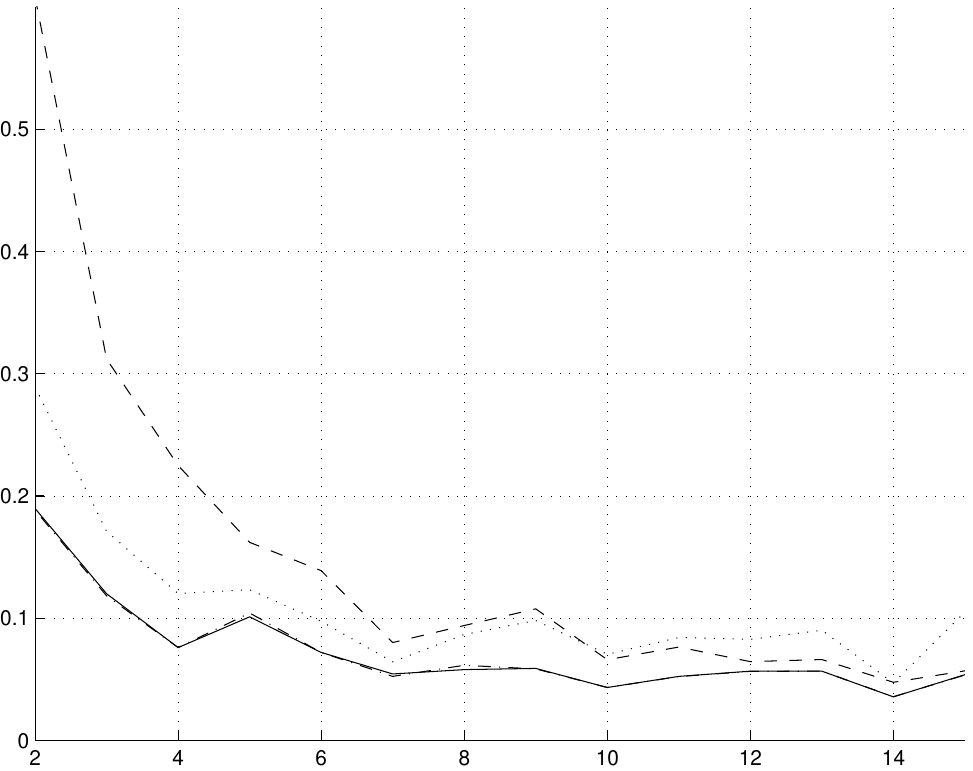}}
\end{tabular}
}
\caption{Calibration errors with
respect to the number of calibration motions using
small motions (Same conventions as in Fig.~\ref{cmp_toutes_bruit})}
\label{cmp_mvts}
\end{figure*}

\subsection{Experiments on real data}

When dealing with real data, no ground-truth value is available for
comparison. Therefore, we defined two measures of accuracy.
The first one compares, for
each motion $i$, $\matx{A}_i \matx{X}$ and $\matx{X} \matx{B}_i$. We
then gathered all these errors into RMS errors.
The second one is a relative error with respect to a reference
estimation of the
hand-eye transformation, computed by a RANSAC method as
follows. From a set of recorded 33 reference positions, we chose 100
sets of 11 positions (10 motions) among the 33 available, then we
computed from them 100 estimations of the hand-eye transformation and
finally we computed the ``mean'' rigid transformation of these estimations
(see~\citep{mathese} for details).

The following two experiments follow the same procedure. Firstly, a
single trajectory is recorded. Secondly, camera and robot motions are
estimated. Thirdly, the hand-eye transformation is estimated using
several methods. Our linear methods are used in their original form,
without any RANSAC method. Finally, each estimation is
measured as described above.

\begin{figure}[t]
\centerline{
\includegraphics[width=0.95\linewidth]{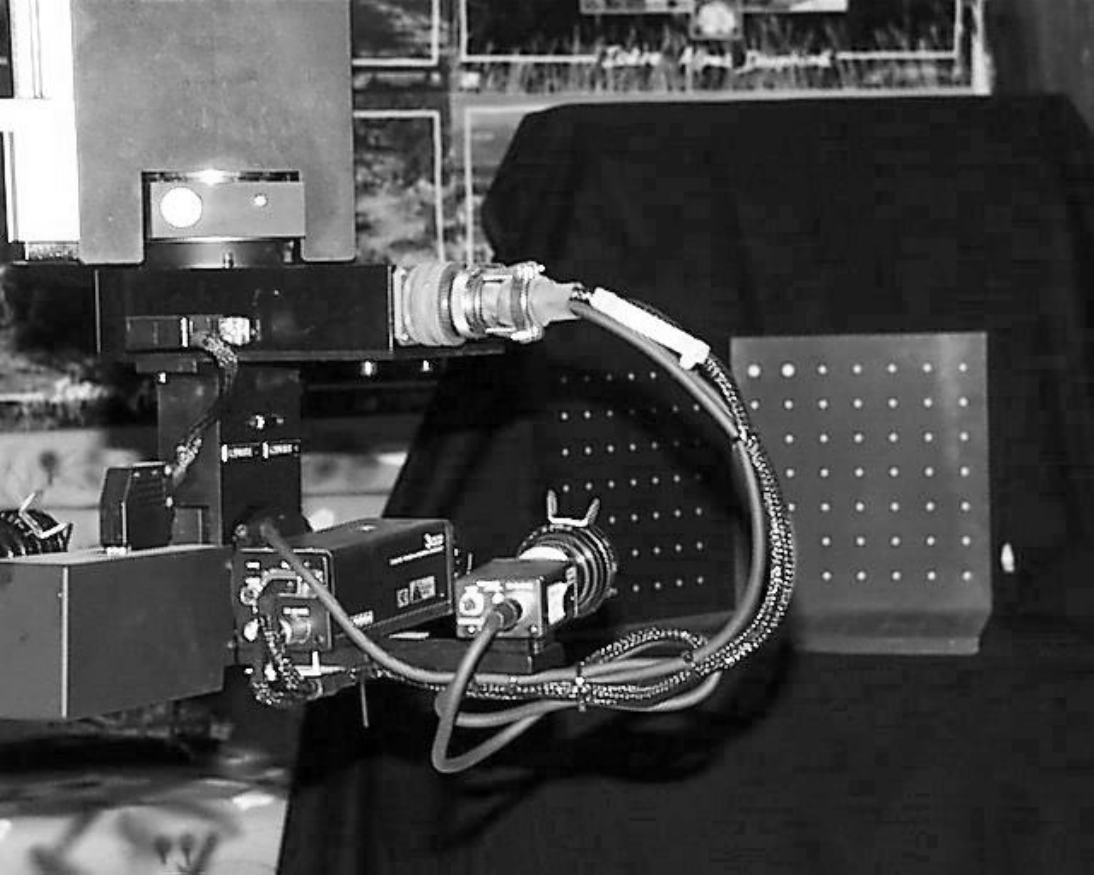}}
\caption{\label{mire} In Experiment 1, the camera observes a
calibration grid.}
\end{figure}

\subsubsection{Experiment 1}
To evaluate the correctness of the solution obtained by
hand-eye self-calibration, we had to compare it with those obtained by
classical calibration methods with the same data. 

Hence, we took images of our
calibration grid (Fig.~\ref{mire})  and performed hand-eye calibration with the
axis/angle method~\citep{tsai89} ({\bf M1}), the dual quaternion
method~\citep{daniilidis96b} ({\bf M2}), the non-linear
minimization~\citep{horaud95e} ({\bf M3}) and the linear formulation  ({\bf M4}).
Finally, using the same points, extracted from the images of the
calibration grid, but not their 3{\sc d} model, we applied the
hand-eye self-calibration method ({\bf M5}). The Euclidean \ddd
reconstruction method we used is the one proposed in~\citep{christy96a}.

In this experiment, we defined the trajectory by ordering the 33
reference positions so that they were as far as
possible from each other according to the advice given
in~\citep{tsai89}. The RMS errors between $\matx{A}_i \matx{X}$ and
$\matx{X} \matx{B}_i$ are displayed in
Fig.~\ref{expe1}. 
It can be seen that ({\bf M4}) gives the smallest
error in rotation due to the numerical efficiency of the SVD and thus obtains
also a reduced error in translation. As for ({\bf M5}), it gives
larger errors, as expected since the 3{\sc d} model is not
used. However, the degradation is rather small and can be explained by 
an approximative estimation of the intrinsic parameters.

\begin{figure}[t]
\centerline{
 \resizebox{!}{6cm}{
  \includegraphics{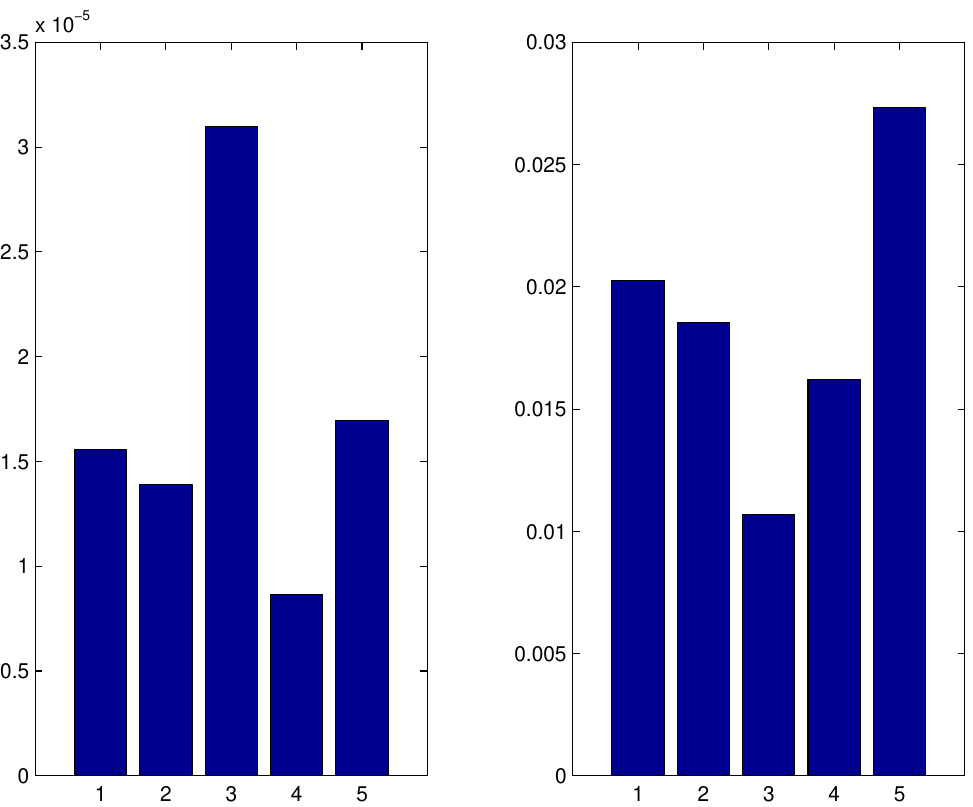}
                   }
           }
\caption{RMS errors in rotation (left) and translation
(right) with 33 images of a calibration grid for each method (see text).}
\label{expe1}
\end{figure}

Then, we compared the estimations obtained above to the reference
estimation. We gathered the errors in Table~\ref{table1}. It confirms
that the linear method is numerically very efficient, especially as
far as rotation is concerned.
Moreover, the self-calibration method yields a lower accuracy, which
nevertheless remains acceptable in the context of visual servoing~\citep{espiau93}.

\begin{table}[h!]
\caption{\label{table1} Comparison with a reference estimation of the hand-eye transformation.}
\centerline{
\begin{tabular}{|c|c|c|}
\hline
Method&Rotation error&Translation error\\
\hline
{\bf M1}&$1.10.10^{-5}$&0.018\\
{\bf M2}&$1.61.10^{-5}$&0.096\\
{\bf M3}&$9.77.10^{-5}$&0.149\\
{\bf M4}&$0.06.10^{-5}$&0.023\\
{\bf M5}&$1.99.10^{-5}$&0.322\\
\hline
\end{tabular}}
\end{table}

\begin{figure}[t]
\centerline{
\includegraphics[width=0.49\linewidth]{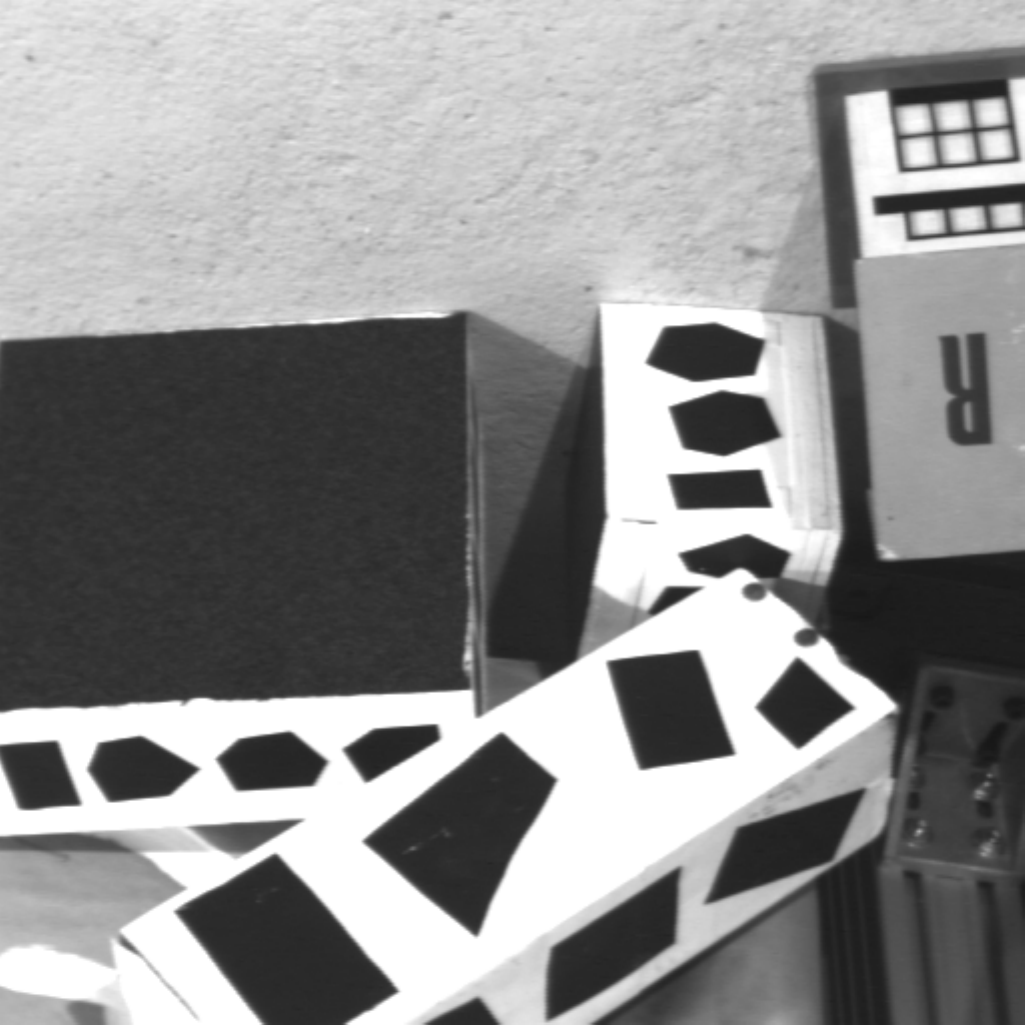}
\includegraphics[width=0.49\linewidth]{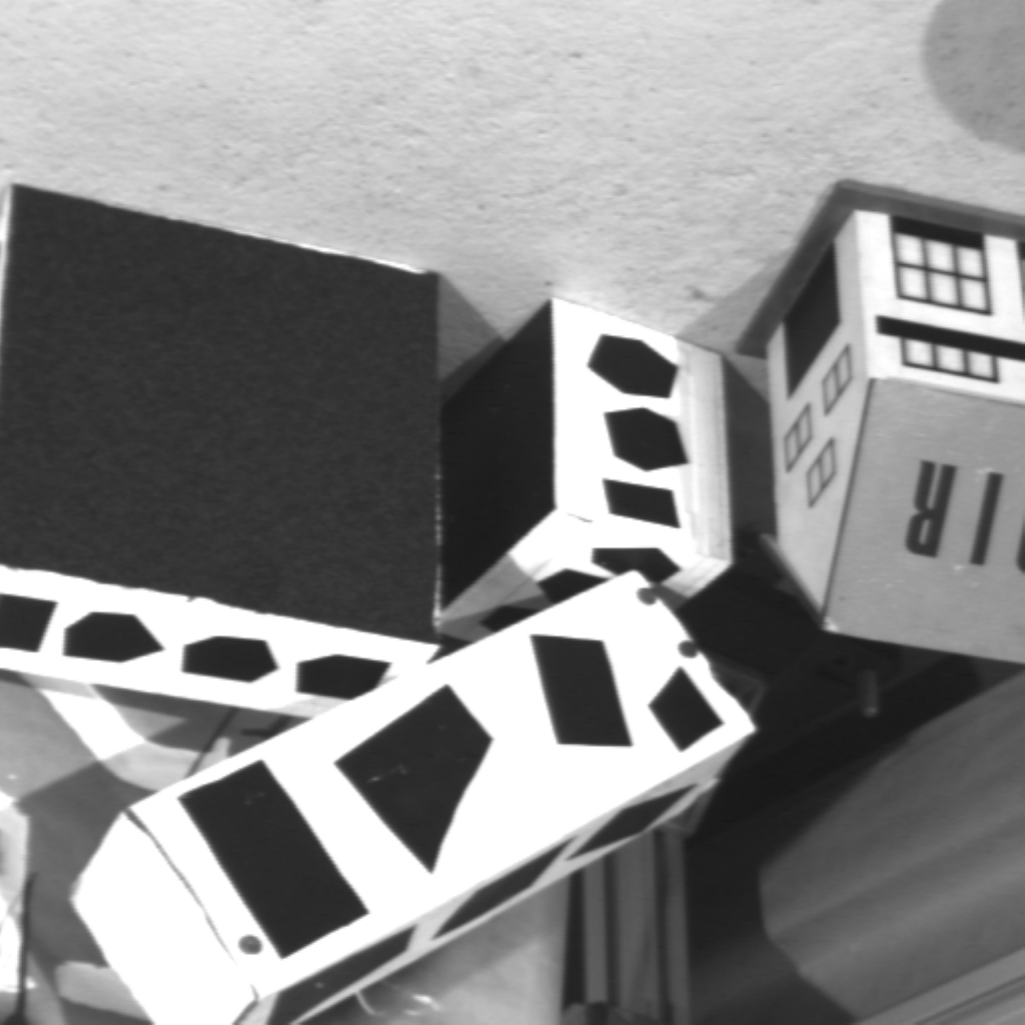}
}
\centerline{
\includegraphics[width=0.49\linewidth]{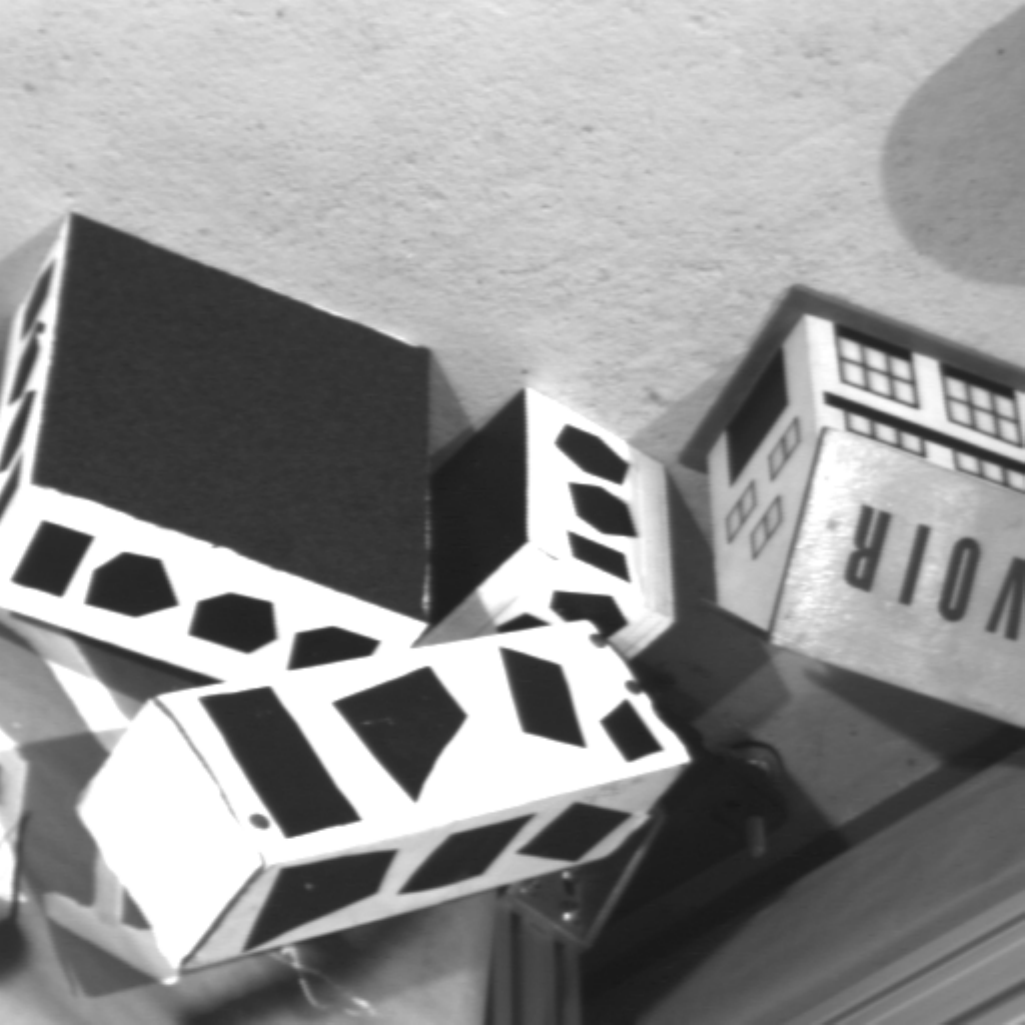}
\includegraphics[width=0.49\linewidth]{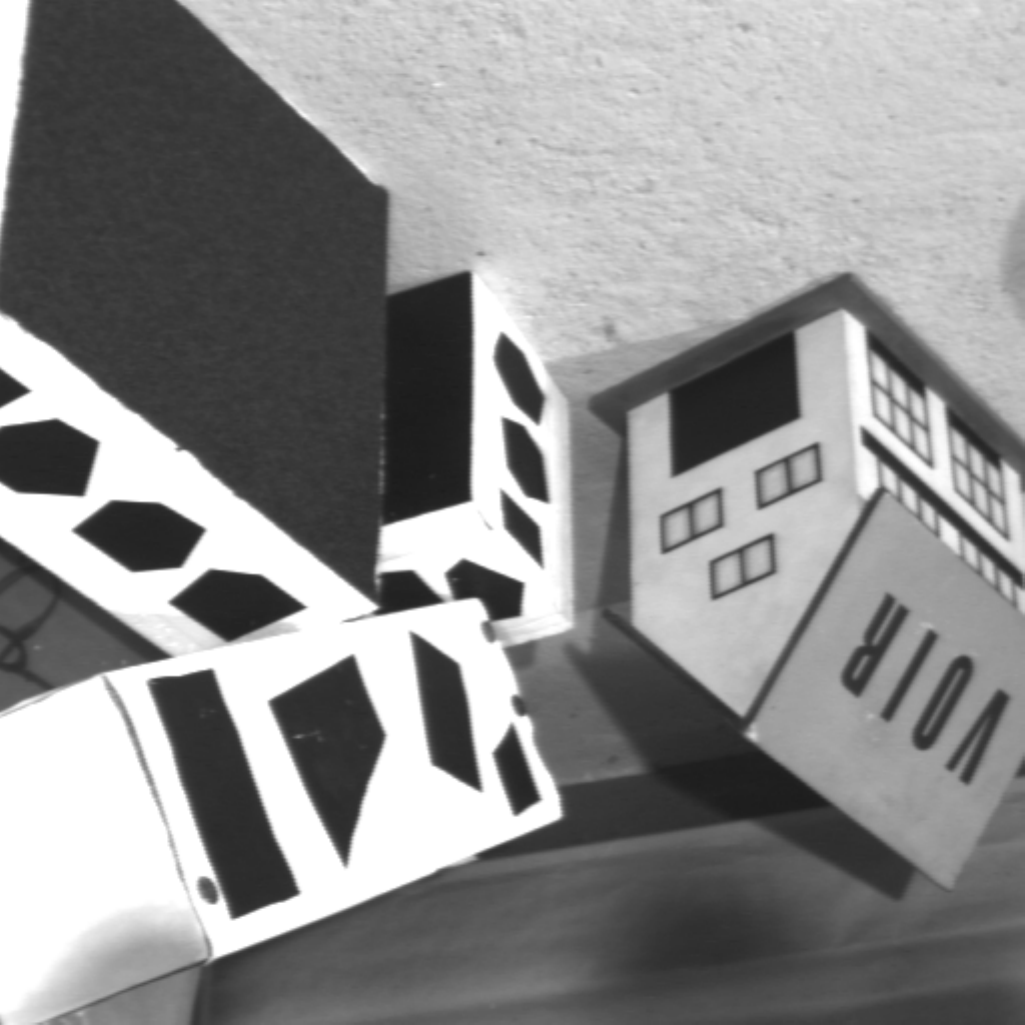}
}
\caption{Four images from a sequence used for
hand-eye self-calibration in Experiment 2.}
\label{images}
\end{figure}

\subsubsection{Experiment 2}
In a second experiment, we tested ({\bf M5}) with more realistic
images. Four positions were defined where the images shown in
Fig.~\ref{images} were taken. In the first image, points were
extracted and then tracked during the motion between each position of
the camera. Then, hand-eye self-calibration was performed upon the tracked
points.

In a goal of comparison, the blocks were replaced by the calibration grid and
the robot was moved anew to the four predefined positions. Then, hand-eye
calibration was performed with the images taken there.

The RMS errors between  $\matx{A}_i \matx{X}$ and
$\matx{X} \matx{B}_i$ in this experiment are given in Fig.~\ref{expe2}. They
show an awful behavior of the non linear minimization method,
probably due to the small number of data. They also show a slightly higher
degradation of the performance of ({\bf M5}) compared to the
others. Nevertheless, it remains in an acceptable ratio since the
relative error in translation is close to 3\%. 

We also compared the results
obtained in this experiment to the reference estimation
(Table~\ref{table2}). This comparison confirms the 
accuracy of both the linear method and the self-calibration scheme.

\begin{figure}[t]
\centerline{
 \resizebox{!}{6cm}{
  \includegraphics{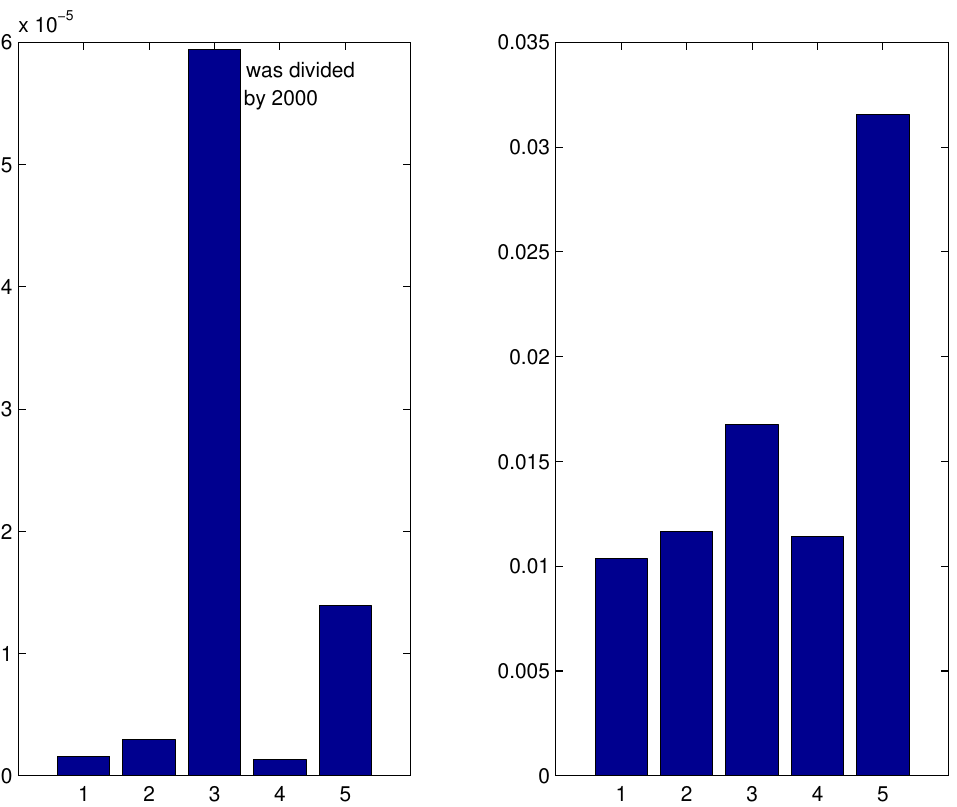}
                   }
           }
\caption{RMS errors in rotation (left) and translation
(right) with 4 images (see text).}
\label{expe2} 
\end{figure}

\begin{table}[h!]
\caption{\label{table2} Comparison with a reference estimation of the hand-eye transformation.}
\centerline{
\begin{tabular}{|c|c|c|}
\hline
Method&Rotation error&Translation error\\
\hline
{\bf M1}&$2.7.10^{-5}$&0.18\\
{\bf M2}&$2.8.10^{-5}$&0.22\\
{\bf M3}&1.82&1.01\\
{\bf M4}&$2.3.10^{-5}$&0.17\\
{\bf M5}&$2.8.10^{-4}$&0.20\\
\hline
\end{tabular}}
\end{table}

\section{Conclusion}
\label{conclusion}
We proposed a hand-eye self-calibration method which reduces
the human supervision compared with classical calibration methods.
The cost of releasing the human constraint is a small degradation of
the numerical accuracy. However, the obtained precision is good enough 
in the context of visual servoing.

This method is based on the structure-from-motion paradigm, rather
than pose estimation, to compute the camera motions and its
derivation includes a new linear formulation of hand-eye calibration.
The linearity of the formulation allows a simple algebraic
analysis. Thus, we  
determined the parts of the hand-eye transformation that can be
obtained from a reduced number of motions which does not allow a
complete calibration.
Moreover, the linear formulation provides improved numerical accuracy
even in the case where the camera/robot rotations have small
amplitude. 

However, one difficulty with the Euclidean \ddd reconstruction with a
moving camera is to be able to find reliable point correspondences
between images. 
The method proposed in \citep{christy96a} solves this problem by
tracking points along the motion. However, it requires that the points are
tracked from the beginning until the end of the robot trajectory.
This is a hard constraint since, in practice, one hardly obtains
enough points after a long trajectory.

Stereo-vision may offer the answer to this problem since it was shown  that
Euclidean reconstruction can be performed, without any prior knowledge, from two
Euclidean motions of a stereo pair~\citep{devernay96a}. This is fully in
coherence with our constraints. Moreover, this kind of method releases
the constraint on the presence of points along the whole sequence of
images. Exploiting this idea, Ruf {\em et al.} proposed a method
for projective robot kinematic calibration and its application to visual
servoing~\citep{ruf99}.

Finally, there is a pending question which was never answered:
``What are the motions for hand-eye (self-)calibration that yield
the higher numerical accuracy~?''

\appendix
\section{Proof of Lemma~\ref{la_solution}}
\label{preuve}

\subsection{Preliminary results}

\begin{result} 
\label{RxR'}
Given two similar rotation matrices $\matx{R}$ and $\matx{R}^{\prime}$ (i.e. there exists
a rotation matrix $\rx$ such that $\matx{R}^{\prime}=\rx
\matx{R} \rx^T$) then

%\begin{enumerate}
%\item 
1) if $\vect{v}$ is an eigenvector
of $\matx{R} \otimes \matx{R^{\prime}}$, then 
\mbox{$(\matx{I} \otimes \rx^T) \vect{v}$}
is an eigenvector of $\matx{R} \otimes
\matx{R}$ for the same eigenvalue;
%\item

2) if $\vect{x}$ is an eigenvector of $\matx{R} \otimes \matx{R}$,
then $(\matx{I} \otimes \rx) \vect{x}$ is an eigenvector of
$\matx{R} \otimes \matx{R^{\prime}}$ for the same eigenvalue.
%\end{enumerate}
\end{result}
\begin{proof}
%\begin{enumerate}
%\item
1) Let $\vect{v}$ be an eigenvector of $\matx{R} \otimes
\matx{R^{\prime}}$ with eigenvalue $\lambda$.
Then, $(\matx{R} \otimes \matx{R^{\prime}}) \vect{v} = \lambda \vect{v}$.
Replacing $\matx{R^{\prime}}$ by $\rx \matx{R} \rx^T$ in
this relation gives:
\[
(\matx{R} \otimes \rx \matx{R} \rx^T) \vect{v} = \lambda
\vect{v}
\]

From $(\matx{A} \otimes \matx{B})(\matx{C} \otimes \matx{D}) = 
(\matx{A} \matx{C}) \otimes (\matx{B} \matx{D})$\citep{Bellman60}, we obtain:
\[
(\matx{I} \otimes \rx) (\matx{R} \otimes \matx{R}) (\matx{I}
\otimes \rx^T) \vect{v} = \lambda \vect{v}
\]

As $(\matx{A} \otimes \matx{B})^{-1} = \matx{A}^{-1} \otimes
\matx{B}^{-1}$\citep{Bellman60}, we derive the following relation:
\[
(\matx{I} \otimes \rx^T)^{-1} (\matx{R} \otimes \matx{R}) (\matx{I}
\otimes \rx^T) \vect{v} = \lambda \vect{v}
\]

Hence, $(\matx{R} \otimes \matx{R}) (\matx{I} \otimes \rx^T) \vect{v} =
\lambda (\matx{I} \otimes \rx^T) \vect{v} $.

%\item
2) Let $\vect{x}$ be an eigenvector of $\matx{R} \otimes \matx{R}$ with
eigenvalue $\alpha$.
Then,
\[
(\matx{R} \otimes \matx{R}) \vect{x} = \alpha \vect{x}
\]
As $(\matx{I} \otimes \rx^T) (\matx{I}
\otimes \rx)=\matx{I}$, we can insert it on both sides:
\begin{small}
\[
(\matx{R} \otimes \matx{R}) (\matx{I} \otimes \rx^T) (\matx{I}
\otimes \rx) \vect{x} = \alpha (\matx{I} \otimes \rx^T) (\matx{I}
\otimes \rx) \vect{x}
\]
\end{small}
which rewrites as:
\begin{small}
\[
(\matx{I} \otimes \rx) (\matx{R} \otimes \matx{R}) (\matx{I}
\otimes \rx^T) (\matx{I}
\otimes \rx) \vect{x} = \alpha (\matx{I} \otimes \rx)
\vect{x}
\]
\end{small}
Hence,
\[
(\matx{R} \otimes \matx{R^{\prime}}) (\matx{I} \otimes \rx)
\vect{x} = \alpha (\matx{I} \otimes \rx) \vect{x}.
\]
%\end{enumerate}
\end{proof}

%\begin{proposition}
%\label{RxRI=I}
%If $\matx{R}$ is a rotation matrix, then the 9-vector $vec(\matx{I}_3)$ is an eigenvector of $\matx{R} \otimes
%\matx{R}$ with eigenvalue 1.
%\end{proposition}
%\begin{proof}
%When expanding the product $\matx{R} \otimes \matx{R}\
%vec(\matx{I}_3)$, we obtain the scalar products of the rows of
%$\matx{R}$ which are either 1 or 0.
%\end{proof}

\begin{result}
\label{R1R2=>I}
Let $\matx{R}_1$ and $\matx{R}_2$ be 2 rotation matrices with non
parallel axes and $\matx{R}_i, 2 < i \le n$, any rotation matrices. 
Let $\matx{R}$ be another rotation matrix. Then, 
\[
\left.
\begin{array}{c}
\matx{R}_1 \otimes \matx{R}_1\ vec(\matx{R}) = vec(\matx{R})\\
\matx{R}_2 \otimes \matx{R}_2\ vec(\matx{R}) = vec(\matx{R})\\
\vdots\\
\matx{R}_n \otimes \matx{R}_n\ vec(\matx{R}) = vec(\matx{R})
\end{array}
\right\}
\Rightarrow \matx{R}=\matx{I}_3
\]
\end{result}
\begin{proof}
The first two equations of the previous system are equivalent to
\[\begin{array}{l}
\matx{R}_1 \matx{R} = \matx{R} \matx{R}_1\\
\matx{R}_2 \matx{R} = \matx{R} \matx{R}_2
\end{array}\]
If $\matx{R}$ satisfies the first equation, then either $\matx{R}$ is
the identity or it has the same rotation axis as $\matx{R}_1$.
Similarly, it is either the identity or has the same rotation axis as
$\matx{R}_2$.
As $\matx{R}_1$ and $\matx{R}_2$ have different rotation axes, it must
be the identity.

With the same logic, adding
other rotations $\matx{R}_i, 2 < i \le n$ such that 
$\matx{R}_i \otimes \matx{R}_i\ vec(\matx{R}) = vec(\matx{R})$ does not change the
conclusion of the above result concerning $\matx{R}$.
\end{proof}

\begin{result}
\label{R1R2M=>sI}
Let $\matx{R}_1$ and $\matx{R}_2$ be two rotation matrices with non
parallel rotation axes and $\matx{R}_i, 2 < i \le n$, any rotation matrices.
Let \mbox{$\matx{M}\neq 0$} be a matrix such that
\[
\begin{array}{c}
\matx{R}_1 \otimes \matx{R}_1\ vec(\matx{M}) = vec(\matx{M})\\
\matx{R}_2 \otimes \matx{R}_2\ vec(\matx{M}) = vec(\matx{M})\\
\vdots\\
\matx{R}_n \otimes \matx{R}_n\ vec(\matx{M}) = vec(\matx{M})\\
\end{array}
\]
Then, 
\[ \exists \mu \neq 0, \matx{M}=\mu \matx{I}_3
\]
\end{result}
\begin{proof}
To write
\[ 
\matx{R}_1 \otimes \matx{R}_1\ vec(\matx{M}) = vec(\matx{M})
\]
is equivalent to say that $\matx{R}_1$ and $\matx{M}$ commute.
Therefore, $\matx{M}$ is of the form $\mu \matx{R}$ where
$\mu \neq 0$ and $\matx{R}$ is a rotation matrix which commutes with
$\matx{R}_1$. This can be easily seen by replacing $\matx{M}$ by its SVD.

Thus, $\matx{M} = \mu \matx{R}$ where $\matx{R}$ is such that:
\begin{equation}
\begin{array}{c}
\matx{R}_1 \matx{R} = \matx{R} \matx{R}_1\\
\matx{R}_2 \matx{R} = \matx{R} \matx{R}_2\\
\vdots\\
\matx{R}_n \matx{R} = \matx{R} \matx{R}_n\\
\end{array}
\end{equation}

From Preliminary result~\ref{R1R2=>I}, we obtain $\matx{R}=\id{3}$ and
$\matx{M} = \mu \id{3}$.
\end{proof}

\subsection{Proof of Lemma~\ref{la_solution}} 

System~(\ref{pour_svd}) is equivalent to
\begin{equation}
\begin{array}{c}
\ra_1 \otimes \rb_1 \vect{v} = \vect{v}\\
\ra_2 \otimes \rb_2 \vect{v} = \vect{v}\\
\vdots\\
\ra_n \otimes \rb_n \vect{v} = \vect{v}
\end{array}
\end{equation}
Under the assumption that the camera motions and the robot motions are 
rigidly linked by a fixed hand-eye transformation
($\rx, \tx$) and from Preliminary result~\ref{RxR'}, this
system becomes:
\begin{equation}
\begin{array}{c}
\rb_1 \otimes \rb_1 \vect{v}' = \vect{v}'\\
\rb_2 \otimes \rb_2 \vect{v}' = \vect{v}'\\
\vdots\\
\rb_n \otimes \rb_n \vect{v}' = \vect{v}'
\end{array}
\end{equation}
where $\vect{v}' = (\matx{I} \otimes \rx^T) \vect{v}$. 
Applying the result of Preliminary result~\ref{R1R2M=>sI}, we obtain that
$vec^{-1}(\vect{v}') = \mu \matx{I}_3$. Using the definition of
$\vect{v}'$ and the properties of the Kronecker product, we end up with:
\[
vec^{-1}(\vect{v}) \rx^T= \mu \id{3}
\]
where $\matx{V}=vec^{-1}(\vect{v})$. Hence,
\[
\matx{V}= \mu \rx
\]
Consequently, the matrix $\matx{V}$ extracted from the null space
of~(\ref{pour_svd}) is proportional to the hand-eye rotation.
The coefficient $\mu$ is obtained from the orthogonality
constraint: $det(\rx)=1$. The latter becomes
\mbox{$det(\matx{V})=\mu^3$} which finally gives:
\[
\mu = sgn(det(\matx{V}))\ |det(\matx{V})|^{1/3}
\]
\begin{flushright}
$\Box$
\end{flushright}
\section*{Acknowledgments}
This work was supported by the European Community through the Esprit-IV reactive LTR project number 26247 (VIGOR). During this work, Nicolas Andreff was a Ph.D. candidate at INRIA Rh\^one-Alpes.
\balance
\bibliographystyle{spbasic}
%\bibliography{andreff2001}

\end{document}